\documentclass[12pt, twoside]{article}
\usepackage{amsmath,amsthm,amssymb}
\usepackage{times}
\usepackage{enumerate}
\usepackage[numbers]{natbib}

\def\titlerunning#1{\gdef\titrun{#1}}
\makeatletter
\def\author#1{\gdef\autrun{\def\and{\unskip, }#1}\gdef\@author{#1}}
\def\address#1{{\def\and{\\\hspace*{18pt}}\renewcommand{\thefootnote}{}%
\footnote {#1}}%
\markboth{\autrun}{\titrun}}
\makeatother
\def\email#1{e-mail: #1}
\def\subjclass#1{{\renewcommand{\thefootnote}{}%
\footnote{\emph{Mathematics Subject Classification (2010):} #1}}}
\def\keywords#1{\par\medskip
\noindent\textbf{Keywords.} #1}

\numberwithin{equation}{section}

\usepackage{amsmath}
\usepackage{amsthm}
\usepackage{amssymb}
\usepackage{graphicx}
\usepackage{esint}
\usepackage{subfig}
\usepackage{color}
\usepackage[english]{babel}
\usepackage{enumerate}
\usepackage{color,soul}
\usepackage{bbm}
\usepackage{multirow}
\usepackage{mathtools}
\usepackage{hyperref}
\usepackage{parskip}
\usepackage{thm-restate}
\usepackage{microtype}

\newtheorem{theorem}{Theorem}[section]
\newtheorem{lemma}[theorem]{Lemma}
\newtheorem{proposition}[theorem]{Proposition}
\newtheorem{remark}[theorem]{Remark}
\newtheorem{corollary}[theorem]{Corollary}
\newtheorem{definition}[theorem]{Definition}
\newtheorem{example}[theorem]{Example}

\DeclareMathOperator*{\diam}{diam}
\DeclareMathOperator*{\diag}{diag}
\DeclareMathOperator*{\id}{id}

\newcommand{\tr}{\text{Trace}}
\newcommand{\R}{\mathbb{R}}

\newcommand{\N}{\mathbb{N}}

\newcommand{\Xcal}{\mathcal{X}}
\newcommand{\Ycal}{\mathcal{Y}}
\newcommand{\Fcal}{\mathcal{F}}
\newcommand{\Dcal}{\mathcal{D}}

\newcommand{\Gcal}{\mathcal{G}}
\newcommand{\Rcal}{\mathcal{R}}

\newcommand{\Hcal}{\mathcal{H}}

\newcommand{\Acal}{\mathcal{A}}

\newcommand{\lip}{\mathrm{Lip}}
\newcommand{\relu}{\mathrm{ReLU}}
\newcommand{\sigmoid}{\mathrm{Sig}}
\newcommand{\TV}{\mathrm{TV}}
\newcommand{\Hcalode}{\mathcal{H}_{\text{ode}}}

\newcommand{\ch}{\mathrm{CH}}
\newcommand{\chbar}{\overline{\mathrm{CH}}}
\newcommand{\F}{\mathcal F}

\renewcommand{\i}{\mathbf{i}}

\begin{document}


\baselineskip=17pt


\titlerunning{Deep Learning via Dynamical Systems: An Approximation Perspective}

\title{Deep Learning via Dynamical Systems: An Approximation Perspective}

\author{
    Qianxiao Li \and Ting Lin \and Zuowei Shen
}

\date{}

\maketitle

\address{
    Q. Li: Department of Mathematics, National University of Singapore; \email{qianxiao@nus.edu.sg}
    \and
    T. Lin: School of Mathematical Sciences, Peking University; \email{lintingsms@pku.edu.cn}
    \and
    Z. Shen:  Department of Mathematics, National University of Singapore; \email{matzuows@nus.edu.sg}
}



\begin{abstract}
    We build on the dynamical systems approach to deep learning, where deep residual networks are idealized as continuous-time dynamical systems, from the approximation perspective. In particular, we establish general sufficient conditions for universal approximation using continuous-time deep residual networks, which can also be understood as approximation theories in $L^p$ using flow maps of dynamical systems. In specific cases, rates of approximation in terms of the time horizon are also established.
    Overall, these results reveal that composition function approximation through flow maps present a new paradigm in approximation theory and contributes to building a useful mathematical framework to investigate deep learning.
\end{abstract}

\section{Introduction and Problem Formulation}
\label{sec:intro}

Despite the empirical success of deep learning, one outstanding challenge is to develop a useful theoretical framework to understand its effectiveness by capturing the effect of sequential function composition in deep neural networks. In some sense, this is a distinguishing feature of deep learning that separates it from traditional machine learning methodologies.

One candidate for such a framework is the \emph{dynamical systems approach}~\cite{weinan2017proposal,Haber2017,li2017maximum}, which regards deep neural networks as a discretization of an ordinary differential equation. Consequently, the latter can be regarded as the object of analysis in place of the former. An advantage of this idealization is that a host of mathematical tools from dynamical systems, optimal control and differential equations can then be brought to bear on various issues faced in deep learning, and more importantly, shed light on the role of composition on function approximation and learning.

Since its introduction, the dynamical systems approach led to much progress in terms of novel algorithms~\cite{Parpas2019,Zhang2019}, architectures~\cite{Haber2017,Chang2018,Ruthotto2018,pmlr-v80-lu18d,Wang2018,Tao2018,Sun2018} and emerging applications~\cite{Zhang2018a,Zhang2018b,Chen2018,Lu2019,effland2020variational}. On the contrary, the present work is focused on the theoretical underpinnings of this approach. From the optimization perspective, it has been established that learning in this framework can be recast as an mean-field optimal control problem~\cite{li2017maximum,weinan2019mean}, and local and global characterizations can be derived based on generalizations of the classical Pontryagin's maximum principle and the Hamilton-Jacobi-Bellman equation.
Other theoretical developments include continuum limits and connections to optimal transport~\cite{Sonoda2017,Thorpe2018,Sonoda2019}.
Nevertheless, the other fundamental questions in this approach remain largely unexplored, especially when it comes to the function approximation properties of these continuous-time idealizations of deep neural networks. In this paper, we establish some basic results in this direction.

\subsection{The Supervised Learning Setting}
We first describe the setting of the standard supervised learning problem we study in this paper. We consider a set of inputs $\Xcal \subset \R^n$ and outputs $\Ycal \subset \R^m$ that are subsets of Euclidean spaces. In supervised learning, we seek to approximate some ground truth or target (oracle) function, which is a mapping $F : \Xcal \rightarrow \Ycal$. For example, in a typical classification problem, each $x$ specifies the pixel values of a $d\times d$ image ($n=d^2$), and $y$ is its corresponding class label, which is a one-hot encoding corresponding to $m$ different classes of images, e.g. for $m=3$, $y=(0, 1, 0)$ corresponds to a label belong to the second class. The ground truth function $F$ defines the label $y=F(x)$ associated with each image $x$, and it is the goal of supervised learning to approximate $F$ from data. Concretely, one proceeds in two steps.

First, we specify a hypothesis space
\begin{align}
    \Hcal =
    \left\{
        F_{\theta} : \Xcal \rightarrow \Ycal
        ~|~
        \theta \in \Theta
    \right\},
\end{align}
which is a family of approximating functions parametrized by $\theta \in \Theta$. The parameter set $\Theta$ is usually again some subset of a Euclidean space. For example, in classical linear basis regression models ($\Ycal=\R$), we may consider a set of orthonormal basis functions $\{\phi_i \in L^2(\Xcal), i=1,2,\dots\}$ forming a hypothesis space by linear combinations, i.e.
$
    \Hcal =
    \left\{
        \sum_i a_i \phi_i:
        a_i \in \R, \sum_i a_i^2 < \infty
    \right\}
$. Of course, there are other hypothesis spaces one can consider, such as deep neural networks that we will discuss later.

Next, we find an approximant $\hat{F} \in \Hcal$ of $F$ by solving an optimization problem typically of the form
\begin{align}\label{eq:risk_minimization}
    \inf_{G\in\Hcal} \int_{\Xcal}
    \ell (F(x), G(x)) d\mu(x),
\end{align}
Here, $\mu$ is a probability measure on $\Xcal$ modelling the input distribution and $\ell :\Ycal \times \Ycal \rightarrow \R$ is a loss function that is minimized when its arguments are equal. A common choice for regression problems is the square loss, $\ell(y,y') = \| y-y' \|_2^2$, in which case the solution of~\eqref{eq:risk_minimization} is a projection of $F$ onto $\Hcal$ in $L^2(\Xcal,\mu)$.
For classification problems, typically one uses a surrogate loss function in place of the classification accuracy, e.g. cross entropy loss. Due to non-square loss functions and complex model architectures, in practice problem~\eqref{eq:risk_minimization} is only solved approximately to give a $\hat{F}$ as an approximation to $F$. Moreover, one typically do not have an explicit form for $\mu$, but we have data samples from it: $x_i$, $y_i = F(x_i)$ for $i=1,2,\dots,N$. In this case, we can set $\mu$ to be the empirical measure $\mu = \frac{1}{N}\sum_{i=1}^{N} \delta_{x_i}$, yielding the so-called empirical risk minimization problem
\begin{align}
    \inf_{G\in\Hcal} \frac{1}{N}
    \sum_{i=1}^{N}
    \ell (F(x_i), G(x_i)).
\end{align}
This objective function is also called the training loss, since it is the loss function evaluated on the model predictions versus the true labels, averaged over the training samples.

\subsection{Deep Residual Neural Networks}

In deep learning, the hypothesis space $\Hcal$ consists of functions in the forms of neural networks of varying architectures. In this paper, we will focus on a very successful class of deep network architectures known as residual networks~\cite{He2016}. These neural networks build the hypothesis space by iterating the following difference equation
\begin{equation}\label{eq:difference_eqn_resnet}
    z_{s+1} = z_{s} + f_{\theta_s} (z_{s}),
    \qquad
    z_{0} = x,
    \qquad
    s=0,\dots,S-1.
\end{equation}
The number $S$ is the total number of layers of the network, and $s$ indexes the layers. The function $f_{\theta_s}$ specifies the architecture of the network, which depends on the trainable parameters $\theta_s$ at each layer $s$. For example, in the simplest case of a fully connected network, we have
\begin{align}\label{eq:fully_connected_architecture}
    \begin{split}
        &f_{\theta_s} (z) = V_s \sigma(W_s z + b_s), \\
        &\theta_s = (W_s, V_s, b_s)
        \quad
        W_s \in \R^{q \times n},
        \quad
        V_s \in \R^{n \times q},
        \quad
        b_s \in \R^{q}.
    \end{split}
\end{align}
Here, $\sigma : \R \rightarrow \R$ is called the activation function, and is applied element-wise to a vector in $\R^n$. Popular examples include the rectified linear unit (ReLU) $\relu(z) = \max(0, z)$, the sigmoid $\sigmoid(z) = 1/(1 + e^{-z})$ and $\tanh(z)$. Depending on the application, more complex $f_\theta$ are employed, such as those involving blocks of fully connected layers or convolution layers. In this paper, we do not make explicit assumption on the form of $f_{\theta_s}$ and consider the general difference equation~\eqref{eq:difference_eqn_resnet} defining the class of residual network architecture. We remark that it is possible to have different parameter dimensions for each layer as is often the case in practice, but for simplicity of analysis we shall take them to have the same dimension and belong to a common parameter set $\Theta$, by possibly embedding in higher dimensional Euclidean spaces.

Now, let us denote by $\varphi_S(x;\theta)$ the mapping $x \mapsto z_S$ via~\eqref{eq:difference_eqn_resnet}. This is the flow map of the difference equation, which depends on the parameters $\theta = (\theta_0,\dots,\theta_{S-1})$.
To match output dimension, we typically introduce another mapping $g$ taken from a family $\Gcal$ of functions from $\R^n$ to $\Ycal \subset \R^m$ at the end of the network (classification or regression layer, as is typically called). We will hereafter call this the \emph{terminal family}, and it is usually simple, e.g. some collection of affine functions. Together, they form the $S$-layer residual network hypothesis space
\begin{align}\label{eq:hyp_resnet}
    \mathcal{H}_{\mathrm{resnet}}(S)
    =
    \{
        g\circ \varphi_S(\cdot; \theta)
        ~|~
        g\in \Gcal, \theta \in \Theta^S
    \}
\end{align}
One can see that this hypothesis space is essentially compositional in nature. First, the functions in $\Hcal$ involves a composition of a flow map $\varphi_S$ and the last layer $g$. Moreover, the flow map $\varphi_S$ itself is a composition of maps, each of which is a step in~\eqref{eq:difference_eqn_resnet}. One challenge in the development of a mathematical theory of deep learning is the understanding the effect of compositions on approximation and learning, due to the lack of mathematical tools to handle function compositions.

\subsection{The Dynamical Systems Viewpoint}

The results in this paper concerns a recent approach introduced in part to simplify the complexity arising from the compositional aspects of the residual network hypothesis space. This is the dynamical systems approach, where deep residual networks are idealized as continuous-time dynamical systems~\cite{weinan2017proposal,li2017maximum,Ruthotto2018,weinan2019mean}. Instead of~\eqref{eq:difference_eqn_resnet}, we consider its continuous-time idealization
\begin{equation}\label{eq:dyn}
    \dot{z}(t) = f_{\theta(t)}(z(t)),
    \qquad
    \theta(t) \in \Theta,
    \qquad
    t\in[0,T],
    \qquad
    z(0) = x.
\end{equation}
That is, we replace the discrete layer numbers $s$ by a continuous variable $t$, which results in a new continuous-time dynamics described by an ordinary differential equation (ODE).
Note that for this approximation to be precise, one would need a slight modification of the right hand side~\eqref{eq:difference_eqn_resnet} into $z_{s} + \delta \cdot f_{\theta_s}(z_s)$ for some small $\delta>0$. The limit $\delta \rightarrow 0$ with $T = S \delta$ held constant gives~\eqref{eq:dyn} with the identification $t \approx \delta s$. Empirical work shows that this modification is justified since for trained deep residual networks, $z_{s+1} - z_{s}$ tends to be small~\cite{Veit2016,Jastrzebski2017}. Consequently, the trainable variables $\theta$ is a now a indexed by a continuous variable $t$. We will assume that each $f_{\theta}$ is a Lipschitz continuous function on $\R^n$, so that~\eqref{eq:dyn} admits unique solutions (see Proposition~\ref{prop:existence-uniqueness}).

As in discrete time, for a terminal time $T>0$, $z(T)$ can be seen as a function of its initial condition $x$, and we denote it by $\varphi_T(\cdot, \theta): \R^n \to \R^n$. The map $\varphi_T$ is known as the Poincar\'{e} map or the flow map of the dynamical system~\eqref{eq:dyn}.
It depends on the parameters $\theta = \{\theta(t)\in \Theta : t\in [0, T]\}$, which is now a function of time. We impose a weak regularity condition of $\theta$ with respect to $t$ by restricting $\theta$ to be essentially bounded, i.e. $\theta \in L^\infty([0,T], \Theta)$.
As a result, we can replace the hypothesis space~\eqref{eq:hyp_resnet} by
\begin{align}\label{eq:hyp_dyn_T}
    \Hcalode(T) =
    \left\{
        g \circ \varphi_T(\cdot, \theta)
        ~|~
        g \in \Gcal,
        \theta \in L^\infty([0,T], \Theta)
    \right\}
\end{align}
with the terminal time $T$ playing the role of depth. In words, this hypothesis space contains functions which are regression/classification layers $g$ composed with flow maps of a dynamical system in the form of an ODE. It is also convenient to consider the hypothesis space of arbitrarily deep continuous-time networks as the union
\begin{align}\label{eq:hyp_dyn}
    \Hcalode = \bigcup_{T>0} \Hcalode(T).
    = \bigcup_{T>0}
    \left\{
        g \circ \varphi_T(\cdot, \theta)
        ~|~
        g \in \Gcal,
        \theta \in L^\infty([0,T], \Theta)
    \right\}
\end{align}
The key advantage of this viewpoint is that a variety of tools from continuous time analysis can be used to analyze various issues in deep learning. This was pursued for example, in~\cite{li2017maximum,pmlr-v80-li18b} for learning algorithms and~\cite{Ruthotto2018,Chang2017,Chang2018} on network stability. In this paper, we are concerned with the problem of approximation, which is one of the most basic mathematical questions we can ask given a hypothesis space. Let us outline the problem below.

\subsection{The Problem of Approximation}

The problem of approximation essentially asks how big $\Hcalode$ is. In other words, what kind of functions can we approximate using functions in~$\Hcalode$? Before we present our results, let us first distinguish the concept of approximation and that of representation.

\begin{itemize}
\item We say that a function $F$ can be represented by $\Hcalode$ if $F \in \Hcalode$.
\item In contrast, we say that $F$ can be approximated by $\Hcalode$ if for any $\varepsilon > 0$, there exists a $\widehat{F} \in \Hcalode$ such that it is close to $F$ up to error $\varepsilon$. Equivalently, $F$ lies in the closure of $\Hcalode$ under some topology.
\end{itemize}
Therefore, representation and approximation are mathematically distinct notions. The fact that some class of mappings cannot be represented by $\Hcalode$ does not prevent it from being approximated by $\Hcalode$ to arbitrary accuracy. For example, it is well-known that flow maps must be orientation-preserving (OP), which are a very small set of functions in the Baire Category sense~\cite{palis1974vector}. At the same time, it is also known that OP diffeomorphisms are dense in $L^p$ in dimensions larger than one~\cite{brenier2003p}. However, what we need here is more than density: the approximation set should have good structure for computation. In this paper, we investigate the density of flow maps with structural constraints.

We will work mostly in continuous time. Nevertheless, it makes sense to ask what the results in continuous time imply for discrete dynamics. After all, the latter is what we can actually implement in practice as machine learning models. Observe that in the reverse direction, $z_{s+1} = z_{s} + \delta f_{\theta_s}(z_s)$ can be seen as a forward Euler discretization of~\eqref{eq:dyn}.
It is well-known that for finite time horizon $T$ and fixed compact domain, Euler discretization has global truncation error in supremum norm of $\mathcal{O}(\delta) = \mathcal{O} (T / S)$ (See e.g.~\cite{Leveque2007Finite}, Ch 5). In other words, any function in $\Hcalode$ can be uniformly approximated by a discrete residual network provided the number of layers $S$ is large enough. Consequently, if a function can be approximated by $\Hcalode$, then it can be approximated by a sufficiently deep residual neural network corresponding to an Euler discretization. In this sense, we can see that approximation results in continuous time have immediate consequences for its discrete counterpart.

\section{Main Results and Implications}

In this section, we summarize our main results on the approximation properties of $\Hcalode$ and discuss their significance with respect to related results in the literature in the direction of approximation theory through the viewpoint of function composition, approximation properties of deep neural networks, as well as controllability problems in dynamical systems. We begin with fixing some notation.

\subsection{Notation}
\label{sec:notation}

Throughout this paper, we adopt the following notation:
\begin{enumerate}
    \item Let $K$ be a measurable subset of $\R^n$. We denote by $C(K)$ the space of real-valued continuous functions on $K$, with norm $\|f\|_{C(K)} = \sup_{x\in K} |f(x)|$. Similarly, for $p\in[1,\infty)$, $L^p(K)$ denotes the space of $p$-integrable measurable functions on $K$, with norm $\|f\|_{L^p(K)} = ( \int_{K} |f(x)|^p dx )^{1/p}$. Vector-valued functions are denoted similarly.
    \item A function $f$ on $K$ is called Lipschitz if $|f(x) - f(x')| \le L|x-x'|$ holds for all $x,x' \in K$. The smallest constant $L$ for which this is true is denoted as $\lip(f)$.
    \item Given a uniformly continuous function $f$, we denote by $\omega_f$ its modulus of continuity, i.e. $\omega_f(r) := \sup_{|x-x'| \le r} |f(x)-f(x')|$.
    \item For any collection $\Fcal$ of functions on $\R^n$, we denote by $\overline{\Fcal}$ its closure under the topology of compact convergence. In other words, $f \in \overline{\Fcal}$ if for any compact $K\subset \R^n$ and any $\varepsilon>0$, there exists $\widehat{f}\in \Fcal$ such that $\| f - \widehat{f} \|_{C(K)} \leq \varepsilon$. As a short form, we will refer to this as approximation closure.
    \item For any collection $\Fcal$ of functions on $\R^n$, we denote by $\ch{(\Fcal)}$ its convex hull
    and $\chbar{(\Fcal)}$ the approximation closure of its convex hull, i.e. closure under the topology of compact convergence.
\end{enumerate}

\subsection{Approximation Results}

Let us begin by slightly simplifying the form of the continuous-time hypothesis space. Let us denote by $\Fcal$ the set of functions that constitute the right hand side of Equation~\eqref{eq:dyn}:
\begin{align}
    \Fcal = \{ f_\theta : \R^n \rightarrow \R^n ~|~ \theta \in \Theta \}.
\end{align}
Consequently, we can denote the family of flow maps generated by $\Fcal$ as
\begin{align}\label{eq:varPhi_defn}
    \varPhi(\Fcal, T) := \{
        x \mapsto z(T)
        ~|~
        \dot{z}(t) = f_{t}(z(t)),\,
        f_{t} \in \Fcal, z(0) = x, t\in[0, T]
    \}
\end{align}
This allows us to write $\Hcalode$ compactly without explicit reference to the parameterization
\begin{align}\label{eq:hyp_dyn_noparam}
    \Hcalode =
    \Hcalode(\Fcal, \Gcal)
    =
    \bigcup_{T>0}
    \{
        g\circ \varphi
        ~|~
        g\in\Gcal, \varphi \in \varPhi(\Fcal, T)
    \}.
\end{align}
We will hereafter call $\Fcal$ a control family, since they control the dynamics induced by the differential equation~\eqref{eq:dyn}. Unless specified otherwise, we assume $\Fcal$ contains only Lipschitz functions, which ensures existence and uniqueness of solutions to the corresponding ODEs (See. Proposition~\ref{prop:existence-uniqueness}). As before, $\Gcal$ is called the terminal family.

The central results in this paper establishes conditions on $\Fcal$ and $\Gcal$ that induce an universal approximation property for $\Hcalode$. To state the results we will need some definitions concerning properties of the control family. The first is the concept of well functions, which plays a fundamental role in constructing approximation dynamics.

\begin{definition}[Well Function]\label{def:well_function}
    We say a Lipschitz function $h: \R^n \rightarrow \R$ is a \emph{well function} if there exists a bounded open convex set $\Omega \subset \R^n$ such that
    \begin{equation}
        \Omega \subset \{ x\in \R^n ~|~ h(x) = 0 \} \subset \overline{\Omega}.
    \end{equation}
    Here the $\overline{\Omega}$ is the closure of $\Omega$ in the usual topology on $\R^n$.

    Moreover, we say that a vector valued function $h:\R^{n} \rightarrow \R^{n'}$ is a well function if each of its component $h_i:\R^{n} \rightarrow \R$ is a well function in the sense above.
\end{definition}
The name ``well function'' highlights the rough shape of this type of functions: the zero set of a well function is like the bottom of a well. Of course, the ``walls'' of this well need not always point upwards and we only require that they are never zero outside of $\overline{\Omega}$.

We also define the notion of restricted affine invariance, which is weaker than the usual form of affine invariance.
\begin{definition}[Restricted Affine Invariance]\label{def:res_affine_inv}
    Let $\Fcal$ be a set of functions from $\R^n$ to $\R^n$. We say that $\Fcal$ is restricted affine invariant if $f\in\Fcal$ implies $D f(A \cdot + b) \in \Fcal$, where $b \in \R^n$ is any vector, and $D$, $A$ are any $n\times n$ diagonal matrices, such that the entries of $D$ are $\pm 1$ or 0, and entries of $A$ are smaller than or equal to 1.
\end{definition}

Now, let us state our main result on universal approximation of functions by flow maps of dynamical systems in dimension $n\ge 2$.

\begin{theorem}[Sufficient Condition for Universal Approximation]\label{thm:main}
    Let $n\geq 2$ and $F:\R^n \rightarrow \R^m$ be continuous.
    Suppose that the control family $\Fcal$ and the terminal family $\Gcal$ satisfies the following conditions
    \begin{enumerate}
        \item For any compact $K\subset \R^n$, there exists a Lipschitz $g \in \Gcal$ such that $F(K) \subset g(\R^n)$.
        \item $\chbar({\mathcal F})$ contains a well function (Definition~\ref{def:well_function}).
        \item $\Fcal$ is restricted affine invariant (Definition~\ref{def:res_affine_inv}).
    \end{enumerate}
    Then, for any $p \in [1,\infty)$, compact $K\subset \R^n$ and $\varepsilon > 0$, there exists $\widehat{F} \in \Hcalode$ such that
    \begin{align}
        \| F - \widehat{F} \|_{L^p(K)} \le \varepsilon.
    \end{align}
\end{theorem}

In the language of approximation theory for neural networks, Theorem~\ref{thm:main} is known as an universal approximation theorem, and $\Hcalode$ satisfying the conditions laid out is said to have the universal approximation property.

Here, the covering condition $F(K) \subset g(\R^n)$ is in some sense necessary. If the range of $g$ does not cover $F(K)$, say it misses an open subset $U\subset F(K)$, then no flow maps composed with it can approximate $F$. Fortunately, this condition is very easy to satisfy. For example, suppose $m=1$ (regression problems), then any linear function $g(x) = w^\top x$ for which $w \neq 0$ suffices, since it is surjective.

The requirement $n\geq2$ is also necessary. In one dimension, the result is actually false, due to the topological constraint induced by flow maps of dynamical systems. More precisely, for $n=1$ one can show that each $\widehat{F} \in \Hcalode$ must be continuous and increasing, and furthermore that its closure also contains only increasing functions. Hence, there is no hope in approximating any function that is strictly decreasing on an open interval. However, we can prove the next best thing in one dimension: any continuous and increasing function can be approximated by a dynamical system driven by the control family $\Fcal$.

\begin{theorem}[Sufficient Condition for Universal Approximation in 1D]\label{thm:1D}
    Let $n=1$. Then, Theorem~\ref{thm:main} holds under the additional assumption that $F$ is increasing.
\end{theorem}
\begin{remark}
    In one dimension, Theorem~\ref{thm:1D} still holds if one replaces the $L^p(K)$ norm by $C(K)$, and furthermore one can relax the restricted affine invariance property to invariance with respect to only $D=\pm 1$ and $A = 1$ in Definition~\ref{def:res_affine_inv}, i.e. we only require symmetry and translation invariance.
\end{remark}

Let us now give some examples of control families satisfying the requirements of Theorems~\ref{thm:main} and~\ref{thm:1D} in order to highlight the general applicability of these results.
\begin{example}[ReLU Networks]
    \label{exa:relu}
    Recall that the fully connected architecture~\ref{eq:fully_connected_architecture} with ReLU activation corresponds to the control family
    \begin{align}\label{eq:relu_ctrl_family}
        \Fcal_{\relu}
        =
        \left\{
            z\mapsto V \relu(W z + b)
            ~|~
            W \in \R^{q\times n},
            V \in \R^{n\times q},
            b \in \R^{q}
        \right\}
    \end{align}
    where $\relu(z)_i = \max(z_i, 0)$ and $q\geq n$.
    It is clear that restricted affine invariance holds.
    Moreover, one can easily construct a well function:
    Let $\Omega = (-1, 1)^{n} \subset \R^n$ be the open cube. Consider the function $h:\R^n\rightarrow\R^n$ whose components are all equal and are given by
    \begin{align}\label{eq:h_omega_example}
        [h(z)]_i
        =
        \frac{1}{2n}
        \sum_{j=1}^{n}
        [
            \relu(-1 - z_j)
            +
            \relu(z_j - 1)
        ],
        \qquad
        i = 1,\dots,n.
    \end{align}
    Clearly, $h \in \ch(\Fcal) \subset \chbar(\Fcal)$ and $h$ is a well function with respect to $\Omega$. Therefore, fully connected residual networks with ReLU activations possesses the universal approximation property as a consequence of our results. Note that this result can be proved using other methods that takes an explicit architectural assumption e.g.~\cite{lin2018resnet,Shen2019Deep,shen2019nonlinear}.
\end{example}

We now discuss examples of some architectural variations that can be handled with our approach. As far as we are aware, such results have not been established in the literature using other means.

\begin{example}[Other Activations]
    \label{exa:sigmoid}
    Let us now discuss how our results can apply just as easily to other network architectures, e.g. with different choice of the activation function. As a demonstration, we consider another commonly used activation function known as the sigmoid activation
    \begin{align}
        \sigmoid(z) = \frac{1}{1+e^{-z}}.
    \end{align}
    in place of the ReLU activation in~\eqref{eq:relu_ctrl_family}. We call this family $\Fcal_{\sigmoid}$. In this case, restricted affine invariance is again obvious. To build a well-function, let us define the scalar soft-threshold function $s:\R\rightarrow\R$
    \begin{align}
        s(z) = \frac{1}{2}
        \min(\max(|z| - 1, 0),1)
    \end{align}
    and for $M,N$ positive integers we define the scalar function
    \begin{align}
        s_{M,N}(z) = \frac{1}{2N}
        \sum_{k=1}^{N}
        \Big[
            \sigmoid(M(-q_k-z)) +
            \sigmoid(M(z-q_k))
        \Big].
    \end{align}
    where $q_k = 1 + (k/N)$. We first show that $s_{M,N}$ can approximate $s$ on any compact subset of $\R$ if $M,N$ are large enough, and it is sufficient to consider the subset to be intervals of the form $[-K, K]$.
    We now estimate $|s(z) - s_{M,N}(z)|$ directly, and by symmetry we only need to check $0\leq z\leq K$. There are three cases:

    Case 1: $0 \leq z < 1$. Here, $s(z)=0$. Then each $z-q_k, -z-q_k\ge -1/N$ and hence $|s_{M,N}(z)| \le \frac{1}{1+\exp(M/N)}$.

    Case 2: $q_{l}\le z<q_{l+1}$ for some $l \geq 0$. Then 1) we have $\sigmoid(M(-q_k-z)) < \frac{1}{1+\exp(M/N)}$ for all $k$; 2) $|1 - \sigmoid(M(z-q_k))| < \frac{1}{1+\exp(M/N)}$ for $k<l$, since $z-q_k>1/N$; 3) $|\sigmoid(M(z-q_k))| < \frac{1}{1+\exp(M/N)}$ for $k>l+1$, since $z-q_k<-1/N$.
    Combining these estimates we have
    \begin{align}
        \left|\frac{l}{2N} - s_{M,N}(z)\right|
        < \frac{1}{2N} + \frac{1}{1+\exp(M/N)},
    \end{align}
    and since $|s(z) - \frac{l}{2N}| < \frac{1}{2N}$, we have $|s(z) - s_{M,N}(z)| < \frac{1}{N} + \frac{1}{1+\exp(M/N)}$.

    Case 3: $z\geq 2$. Here, $s(z) = 1$, and the estimates for case 2 above also hold true. Thus, we have $|1 - s_{M,N}(z)| < \frac{1}{N} + \frac{1}{1+\exp(M/N)}$.

    Combining the cases above, we have for any $K>0$ and $|z|\leq K$,
    \begin{align}\label{eq:st_overall_estimate}
        |s(z) - s_{M,N}(z)| < \frac{1}{N} + \frac{1}{1+\exp(M/N)}.
    \end{align}
    By sending $M = N^2 \rightarrow \infty$, we can make the right hand side arbitrarily small, as required.

    Now, let us define the function $h:\R^n \rightarrow \R^n$ where
    \begin{align}
        [h(z)]_i
        = \frac{1}{n}
        \sum_{j=1}^{n}
        s(z_j),
        \qquad
        i=1,\dots,n.
    \end{align}
    It is clear that $h$ is a well function with respect to the cube $(-1,1)^{n}$. Moreover, by estimate~\eqref{eq:st_overall_estimate} $h$ can be uniformly approximated on any compact subset by $[h_{M,N}]_i = \frac{1}{n} \sum_{j=1}^{n} s_{M,N}(z_j)$, which belongs to $\ch(\Fcal_\sigmoid)$. Thus, $h \in \chbar(\Fcal_\sigmoid)$ (recall the definition of closure with respect to the topology of compact convergence in Sec.~\ref{sec:notation}) and we conclude using our results that continuous fully connected residual networks with sigmoid activations also possess the universal approximation property. Other activations such as $\tanh$ can be handled similarly. Importantly, we can see that in our framework, relatively little effort is required to handle such variations in architecture, but for existing approaches whose proofs rely on explicit architectural choices such as ReLU, this may be much more involved.
\end{example}

\begin{example}[Residual Blocks]
    \label{exa:block}
    As a further demonstration of the flexibility of our results, we can consider another type of variation of the basic residual network, which considers a ``residual block'' with more than one fully connected layer. For example, each block can be of the form
    \begin{align}
        \begin{split}
            &u_{s}
            =
            \sigma (W^{(1)}_s z_s + b^{(1)}_s), \\
            &z_{s+1}
            = z_{s} + V_s \sigma
            (W^{(2)}_s u_{s} + b^{(2)}_s),
        \end{split}
    \end{align}
    where $\sigma$ is some nonlinear activation function applied element-wise. In fact, the original formulation of residual networks has such a block structure, albeit with convolutional layers and ReLU activations~\cite{He2016}.
    Now, the corresponding control family for the idealized continuous-time dynamics is
    \begin{align}\label{eq:block_ctrl_family}
        \begin{split}
            \Fcal_{\mathrm{block}}
            =
            \{
                &z\mapsto V \sigma(W^{(2)}
                \sigma(W^{(1)} z + b^{(1)}) + b^{(2)})
                \\
                &~|~
                V \in \R^{n\times q_2},
                W^{(1)} \in \R^{q_1\times n},
                b^{(1)} \in \R^{q_1},
                W^{(2)} \in \R^{q_2\times q_1},
                b^{(2)} \in \R^{q_2}
            \},
        \end{split}
    \end{align}
    where $q_2,q_1 \geq n$. We now show that we can deduce universal approximation of this family from previous results. As before, restricted affine invariance holds trivially. Thus, it only remains to show that $\chbar(\Fcal_{\mathrm{block}})$ contains a well function.

    To proceed, we may set $q_1=q_2=n$ without loss of generality, since otherwise we can just pad the corresponding matrices/vectors with zeros.
    Now, let us assume that the ``one-layer'' control family
    \begin{align}
        \Fcal_{\sigma}
        =
        \{
            x\mapsto V \sigma(W z + b)
            ~|~
            W \in \R^{n\times n}, b \in \R^n
        \}
    \end{align}
    is such that $\chbar({\Fcal_{\sigma}})$ contains a well function that is non-negative. From the previous examples, we know that this is true for $\sigma=\relu$ or $\sigma=\sigmoid$.
    In addition, we assume that the activation $\sigma$ satisfies a non-degeneracy condition: there exists a closed interval $I \subset \R$ such that its pre-image $\sigma^{-1}(I)$ is also a closed interval. Note that most activations we use in practice satisfy this condition.

    Let us define a control family $\Fcal$, which is a subset of $\Fcal_{\mathrm{block}}$, in which we set $W^{(1)} = I$ and $b^{(1)}= 0$. We also reparameterize the remaining variables as $W^{(2)} \rightarrow \widetilde{W}^{(2)} \widetilde{W}^{(1)}$, $b^{(2)} \rightarrow \widetilde{W}^{(2)} \widetilde{b}^{(1)} + \widetilde{b}^{(2)}$ to obtain the smaller control family
    \begin{equation}
    \begin{split}
        \Fcal
        =
        \{
            &z\mapsto V \sigma(\widetilde{W}^{(2)}[\widetilde{W}^{(1)}
            \sigma(z) + \widetilde{b}^{(1)}] + \widetilde{b}^{(2)})
            \\
            &~|~
            V \in \R^{n\times n},
            \widetilde{W}^{(1)} \in \R^{n\times n},
            \widetilde{b}^{(1)} \in \R^{n},
            \widetilde{W}^{(2)} \in \R^{n\times n},
            \widetilde{b}^{(2)} \in \R^{n}
        \},
    \end{split}
    \end{equation}
    We now show that $\chbar(\Fcal)$ contains a well function.
    Since the activation function is applied element-wise, we may first consider the 1D case, as we have done is the first two examples. Suppose that $s$ is a scalar well function such that $z\mapsto (s(z_1), \dots, s(z_n)) \in \chbar{(\Fcal_{\sigma})}$ and that $s$ is non-negative (see Examples~\ref{exa:relu} and~\ref{exa:sigmoid} for construction). Then we know that $z \mapsto (s(a\sigma(z_1)+b),\dots, s(a\sigma(z_n)+b)) \in \chbar(\Fcal) \subset \chbar(\Fcal_{\mathrm{block}})$ for all $a,b\in\R$. It suffices to verify that $s(a\sigma(\cdot)+b)$ is a scalar well function, for some $a$ and $b$ satisfying suitable conditions. We now show this is the case.
    Take a closed interval $I\subset \R$ such that $\sigma^{-1}(I)$ is also a closed interval. By rescaling and translating, we can take the zero set of $s$ to be the interval $[-1, 1]$. Then, we choose $a$, $b$ such that $z \mapsto a z + b$ maps $I$ to $[-1, 1]$, from which we can deduce that $s(a\sigma(\cdot)+b)$ is also a well function.
    We may now construct a well function in $n$ dimensions analogously to the previous examples
    \begin{align}
        [h(z)]_i
        = \frac{1}{n}
        \sum_{j=1}^{n}
        s(z_j),
        \qquad
        i=1,\dots,n,
    \end{align}
    and by construction $h$ is a well function in $\chbar(\Fcal) \subset \chbar(\Fcal_{\mathrm{block}})$.
    By our results, this again induces the universal approximation property of its corresponding $\Hcalode$.
\end{example}

The above examples serves to illustrate the flexibility of a sufficient condition in deriving universal approximation results for many different architectures. It is possible that some, or even all of these results can be derived using other means (such as reproducing other universal function classes), but such arguments are likely to be involved and more importantly, they have to be handled on a case-by-case basis, lacking a systematic approach such as the one introduced in this paper. We end this section with a remark on a negative example.

\begin{remark}
    Observe that using linear activations $\sigma(z) = z$ constitute a control family which does not contain a well function in $\chbar (\Fcal)$. We also can immediately see that it cannot produce universal approximating flow maps, since the resulting flow maps are always linear functions.
\end{remark}

Let us now discuss the implication of Theorems~\ref{thm:main} and~\ref{thm:1D} in three broad directions: 1) approximation of functions by compositions; 2) approximation theory of deep neural networks and 3) control theory and dynamical systems.

\subsection{Approximation of Functions by Composition}

Let us first discuss our results in the context of classical approximation theory, but through the lens of compositional function approximation. In other words, we will recast classical approximation methods in the form of a compositional hypothesis space (c.f.~\eqref{eq:hyp_dyn_T})
\begin{align}
    \Hcal
    =
    \{
        g\circ\varphi
        ~|~
        g\in\Gcal,
        \varphi\in\varPhi
    \},
\end{align}
which then allows us to compare and contrast with the setting considered in this paper. As before, we call $\Gcal$ the terminal family, and for convenience we will refer to $\varPhi$ as a \emph{transformation family}, to highlight the fact that it contains functions whose purpose is to transform the domain of a function $g$ in order to resemble a target function $F$.

We start with the simplest setting in classical approximation theory, namely linear $N$-term approximation. Here, we consider a fixed, countable collection $\Dcal$ of functions from $\R^n$ to $\R$, which we call a dictionary.
The dictionary is assumed to have some structure so that they are simple to represent or compute. Common examples include polynomials, simple periodic functions such as sines and cosines, as well as other types of commonly seen basis functions. Linear $N$-term approximation takes the first $N$ elements $\phi_1,\phi_2,\dots,\phi_N$ of $\Dcal$ and forms an approximant $\widehat{F}$ of $F$ via their linear combinations
\begin{align}
    \widehat{F}(x)
    = \sum_{i=1}^{N} w_i \phi_i(x),
    \qquad
    w_i \in \R^m,
    \quad
    i=1,2,\dots,N.
\end{align}
From the viewpoint of compositional function approximation, we may express the above hypothesis space by considering the linear terminal family
\begin{align}
    \Gcal(N) =
    \left\{
        x \mapsto \sum_{i=1}^{N} w_i x_i
        ~|~
        w_i \in \R^m,
        \,
        i=1,\dots,N
    \right\}
\end{align}
Then, we obtain the compositional representation of the hypothesis space
\begin{align}\label{eq:comp_linear_nterm}
    \Hcal(N)
    =
    \{
        g\circ \varphi
        ~|~
        g \in \Gcal(N),
        \varphi = (\phi_1,\dots,\phi_N)
    \}.
\end{align}
In other words, the basic $N$-term linear approximation can be recast as a compositional hypothesis space consisting of a linear terminal family and a transformation family containing of just one $N$-dimensional vector-valued function, whose coordinates are the first $N$ elements of the dictionary $\Dcal$. This is called linear approximation, because for two target functions $F_1$ and $F_2$, whose best approximation (in terms of lowest approximation error, see Sec.~\ref{sec:intro}) are $\widehat{F}_1$ and $\widehat{F}_2$ respectively, then the best approximation of $\lambda_1 F_1 + \lambda_2 F_2$ is $\lambda_1 \widehat{F}_1 + \lambda_2 \widehat{F}_2$ for any $\lambda_1,\lambda_2 \in \R$.

Nonlinear $N$-term approximation~\cite{devore1998nonlinear} takes this approach a step further, by lifting the restriction that we only use the first $N$ elements in $\Dcal$. Instead, we are allowed to choose, given a target $F$, which $N$ functions to pick from $\Dcal$. For this reason, the dictionaries $\Dcal$ used in nonlinear approximation can be an uncountable family of functions. In this case, the compositional family is
\begin{align}\label{eq:comp_nonlinear_nterm}
    \Hcal(N)
    =
    \{
        g\circ \varphi
        ~|~
        g \in \Gcal(N),
        \varphi \in \Dcal^N
    \}.
\end{align}
The term ``nonlinear'' highlights the fact that the best approximations for functions do not remain invariant under linear combinations.

In these classical scenarios, the compositional formulations~\eqref{eq:comp_linear_nterm} and~\eqref{eq:comp_nonlinear_nterm} share certain similarities. Most importantly, their transformation families have simplistic structures, and the terminal family is linear, hence also simple. Consequently, one must rely on having a large $N$ in order to form a good approximation. A function can be efficiently approximated via linear approximation (i.e. requiring a small $N$) if its expansion coefficients in $\Dcal$ decays rapidly. On the other hand, efficient approximation through nonlinear approximation relies on the sparsity of this expansion. Nevertheless, in both cases the complexity of their respective hypothesis spaces arise from a large number of linear combination of simple functions.

Let us now contrast our results on $\Hcalode$, which is also in this compositional form (See~\eqref{eq:hyp_dyn_T}). First, our results holds for more general terminal families that are not restricted to linear ones. Second, by looking at the form of $\Hcalode$ and comparing with~\eqref{eq:comp_linear_nterm} and~\eqref{eq:comp_nonlinear_nterm}, we observe that the complexity of $\Hcalode$ arises not due to linear combination of functions, since universal approximation holds despite $\varphi$ having fixed output dimension. Instead, the complexity of $\Hcalode$ arises from compositions, a point which we shall now expand on.

Observe that besides the overall compositional structure of a terminal family and a transformation family, a second aspect of composition is also involved in $\Hcalode$: the transformation family $\varPhi(\Fcal,T)$ is itself generated by compositions of simple functions. To see this, observe that for any flow map of an ODE $\varphi_T$ up to time $T$, we can write it as
\begin{align}\label{eq:flowmap_decomp}
    \varphi_T =
    \varphi_{\tau_M} \circ \varphi_{\tau_{M-1}} \circ \dots \circ \varphi_{\tau_2} \circ \varphi_{\tau_1},
\end{align}
where $\tau_1+\dots+\tau_M = T$, and each $\varphi_{\tau_i}$ represents the portion of the flow map from $t=\sum_{s \leq i-1} \tau_s $ to $t=\sum_{s\leq i} \tau_{s}$ (we set $\tau_0=0$)~\cite{Arnold1973Ordinary}.
By increasing the number of such partitions, each $\varphi_{\tau_i}$ becomes closer to the identity mapping.
More generally, the family of flow maps forms a continuous semi-group under the binary operation of composition, with the identity element recovered when the time horizon of a flow map goes to 0.
Therefore, each member of the transformation family $\varPhi(\Fcal, T)$ can be decomposed into a sequence of compositional mappings, each of which can be made arbitrarily close to the identity, as long as at the same time one increases the number of such compositions. While this decomposition holds true for any flow map of an ODE, the main results in this paper go a step further and show that the universal approximation property holds even when the flow map is restricted to one that is generated by some control family $\Fcal$ verifying the assumptions in Theorem~\ref{thm:main}.

We remark that in classical nonlinear approximation, the dictionary $\Dcal$ could also involve compositions. A prime examples of this is wavelets~\cite{mallat1999wavelet,daubechies1992ten,ron1997affine} where one starts with some template function $\psi$ (mother wavelet) and generates a dictionary by composing it with translations and dilations. For example, in one dimension the wavelet dictionary has the following compositional representation
\begin{align}
    \Dcal
    =
    \{
        \psi \circ T
        ~|~
        T(x) = (x-x_0)/\lambda,
        x_0 \in \R,
        \lambda > 0
    \}.
\end{align}
The main contrasting aspect in $\Hcalode$ is that the compositional transformations are much more complex. Instead of simple translations and dilations, the transformation family in $\Hcalode$ involves complex rearrangement dynamics in the form of a ODE flow that may be adapted to the specific target function $F$ at hand.

In summary, contrary to classical approximation schemes, $\Hcalode$ is built from compositions of functions from a simple terminal family $\Gcal$ and a complex transformation family $\varPhi(\Fcal, T)$, whose members can further be decomposed into a sequence of compositions of functions that are simple in two aspects: they are close to the identity map and they are generated by a potentially simple control family $\Fcal$. Consequently, the complexity of $\Hcalode$ arises almost purely from the process of composition of these simple mappings. In other words, we trade complexity in $T$ (compositions) for $N$ (linear combinations), and can achieve universal approximation even when the transformation family has fixed output dimensions. From this viewpoint, the results in this paper highlights the power of composition for approximating functions.

\subsection{Approximation Theory of Deep Neural Networks}

As discussed previously, the transformation family in $\Hcalode$ consisting of flow maps is highly complex due to repeated compositions. At the same time, however, just like dictionaries in linear and nonlinear approximation, it possesses structure that allows us to carry out approximation in practice. Concretely, recall that each flow map can be decomposed as in~\eqref{eq:flowmap_decomp}, where each component $\varphi_{\tau_i}$ is not only close to the identity, but is close in such a way that the perturbation from identity is constrained by the control family $\Fcal$. Thus, one just need to parameterize each $\varphi_{\tau_i}$ by selecting appropriate functions from $\Fcal$ and then compose them together to form an approximating flow map.
From this viewpoint, the family of deep residual network architectures is a realization of this procedure, by using a one-step forward Euler discretization is approximate each $\varphi_{\tau_i}$.
Concretely,
\begin{align}
    \varphi_{\tau_i} (z)
    =
    z
    +
    \int_{0}^{\tau_i}
    f(z(t)) dt
    \approx
    z + {\tau_i} f(z),
\end{align}
for $\tau_i$ small, which corresponds to the family of deep residual architectures motivated in Sec.~\ref{sec:intro}.
The standard convergence result for Euler discretization~\cite{Leveque2007Finite} allows one to carry approximation results in continuous time to the discrete case.
In view of this, we now discuss our results in the context of approximation results in deep learning.

We start with the continuous-time case. Most existing theoretical work on the continuous-time dynamical systems approach to deep learning focus on optimization aspects in the form of mean-field optimal control~\cite{weinan2019mean,liu2019selection}, or the connections between the continuous-time idealization to discrete time~\cite{thorpe2018deep,Sonoda2017,Sonoda2019}.
The present paper focuses on the approximation aspects of continuous-time deep learning, which is less studied. One exception is the recent work of Zhang et al.~\cite{zhang2019approximation}, who derived some results in the direction of approximation. However, an important assumption there was that the driving force on the right hand side of ODEs (here the control family $\Fcal$) are themselves universal approximators. Consequently, such results do not elucidate the power of composition and flows, since each ``layer'' is already so complex to approximate an arbitrary function, and there is no need for the flow to perform any additional approximation.

In contrast, the approximation results here do not require $\Fcal$, or even $\chbar{(\Fcal)}$, to be universal approximators. In fact, $\Fcal$ can be a very small set of functions, and the approximation power of these dynamical systems are by construction attributed to the dynamics of the flow.
For example, the assumption that $\chbar{(\Fcal)}$ contains a well function does not imply $\Fcal$ that drives the dynamical system is complex, since the former can be much larger than the latter. In the 1D ReLU control family that induces the fully connected network with ReLU activations (See~\eqref{eq:difference_eqn_resnet}), one can easily construct a well function with respect to the interval $\Omega = (q_1,q_2)$ by averaging two ReLU functions: $\frac{1}{2} [\relu(q_1-x) + \relu(x-q_2)]$, but the control family $\Fcal = \{ v \relu (w \cdot + b) \}$ is not complex enough to approximate arbitrary functions without further linear combinations. We have already illustrated this in Examples~\ref{exa:relu},~\ref{exa:sigmoid} and~\ref{exa:block}.


We also note that unlike results in~\cite{zhang2019approximation}, the results here for $n\geq 2$ do not require embedding the dynamical system in higher dimensions to achieve universal approximation. The negative results given in~\cite{zhang2019approximation} (and also~\cite{dupont2019augmented}), which motivated embedding in higher dimensions, are basically on limitation of representation: flow maps of ODEs are orientation preserving (OP) homeomorphisms (See Def.~\ref{def:OP}) and thus can only represent such mappings. However, these are not counter-examples for approximation. For instance, it is known that OP diffeomorphism (and hence OP homeomorphisms) can approximate any $L^\infty$ functions on open bounded domains in dimensions greater than or equal to two~\cite{brenier2003p}.

Although the present paper focuses on the continuous-time idealization, we should also discuss the results here in relation to the relevant work on the approximation theory of discrete deep neural networks. In this case, one line of work to establish universal approximation is to show that deep networks can approximate some other family of functions known to be universal approximators themselves, such as wavelets~\cite{mallat2016understanding} and shearlets~\cite{guhring2019error}. Another approach is to focus on certain specific architectures, such as in~\cite{lu2017expressive, lin2018resnet, zhou2018deep,bao2019approx,daubechies2019nonlinear, e_priori_2019}, which sometimes allows for explicit asymptotic approximation rates to be derived for appropriate target function classes. Furthermore, non-asymptotic approximation rates for deep ReLU networks are obtained in~\cite{shen2019nonlinear, Shen2019Deep}. They are based on explicit constructions using composition, and hence is similar in flavor to the results here if we take an explicit control family and discretize in time.

With respect to all these works, the main difference of the results presented here is that we study sufficient conditions for approximation. In other words, we do not start with an {\it a priori} specific architecture (e.g. the form of the function $f_\theta$, or the type of activation $\sigma$ in~\eqref{eq:difference_eqn_resnet}).
In particular, none of the approximation results we present here depend on reproducing some other basis functions that are known to have the universal approximation property. Instead, we derive conditions on the respective control and terminal families $\Fcal,\Gcal$ that induces the universal approximation property in $\Hcalode$. One advantage of this viewpoint is that we can isolate the approximation power that arises from the act of composition, from that which arises from the specific architectural choices themselves.
As an example, the approximation results in~\cite{lin2018resnet} relies on approximating piecewise constant functions with finitely many discontinuities, hence its proof depends heavily on the ReLU activation. Furthermore, the high dimensional results there requires constructing the proximal grid indicator function, which is not straightforward with activations other than ReLU. We note that for the deep non-residual case, more precise approximation results including non-asymptotic rates for the ReLU architecture can be derived~\cite{shen2019nonlinear, Shen2019Deep}. In contrast, the main results in this paper proceeds in a more general way without assuming certain precise architectures.
Therefore, these results have greater applicability to diverse architectures (See Examples~\ref{exa:relu},~\ref{exa:sigmoid} and~\ref{exa:block}), and perhaps even novel ones that may arise in future deep learning applications.

\subsection{Control Theory and Dynamical Systems}

Lastly, the results here are also of relevance to mathematical control theory and the theory of dynamical systems. In fact, the problem of approximating functions by flow maps is closely related to the problem of controllability in the control theory~\cite{sussmann2017nonlinear}. However, there is one key difference: in the usual controllability problem on Euclidean spaces, our task is to steer one particular input $x_0$ to a desired output value $\varphi(x_0)$. However, here we want to steer the entire set of input values in $K$ to $\varphi(K)$ by the \emph{same} control $\theta(t)$. This can be thought of as an infinite-dimensional function space version of controllability, which is a much less explored area and present controllability results in infinite dimensions mostly focus on the control of partial differential equations~\cite{Chukwu1991,Balachandran2002}.

In the theory of dynamical systems, it is well known that functions represented by flow maps possess restrictions. For example,~\cite{palis1974vector} gives a negative result that the diffeomorphisms generated by $C^1$ vector fields are few in the Baire category sense. Some works also give explicit criteria for mappings that can be represented by flows, such as~\cite{fort1955embedding} in $\mathbb R^2$,~\cite{utz1981embedding} in $\mathbb R^n$, and more recently,~\cite{zhang2009embedding} generalizes some results to the Banach space setting.
However, these results are on exact representation, not approximation, and hence do not contradict the positive results presented in this paper. The results on approximation properties are fewer. A relevant one is~\cite{brenier2003p}, who showed that every $L^p$ mapping can be approximated by orientation-preserving diffeomorphisms constructed using polar factorization and measure-preserving flows. The results of the current paper gives an alternative construction of a dynamical system whose flow also have such an approximation property. Moreover, Theorem~\ref{thm:main} gives some weak sufficient conditions for any controlled dynamical system to have this property. In this sense, the results here further contribute to the understanding of the density of flow maps in $L^p$.

\section{Preliminaries}

In this section, we state and prove where necessary some preliminary results that are used to deduce our main results in the next section.

\subsection{Background Results on Ordinary Differential Equations}
\label{sec:ode}

Throughout this paper, we use some elementary properties and techniques in classical analysis of ODEs. For completeness, we compile these results in this section. The proofs of well-known results that are slightly involved are omitted and unfamiliar readers are referred to~\cite{Arnold1973Ordinary} for a comprehensive introduction to the theory of ordinary differential equations.

Note that the differential equation that generates the transformation family in $\Hcalode$ is of the form
\begin{align}
    \dot{z} = f_t(z), \qquad f_t\in\Fcal, \qquad 0<t\leq T,  \qquad z(0) = z_0,
\end{align}
where $z_0,z(t)\in\R^n$.
Such equations are called time inhomogeneous, since the forcing function $f_t$ changes in time. However, in subsequent proofs of approximation results, we usually consider $t\mapsto f_t$ that are piece-wise constant, i.e. $f_t = f_i$ for all $t\in [T_i, T_{i+1}]$. In this case, for each interval on which $f_t$ does not change, it is enough to consider the time-homogeneous equation
\begin{equation}
\label{eq:ode}
    \dot{z} = f(z), \qquad 0<t\leq T, \qquad z(0) = z_0,
\end{equation}
where $f:\R^n\rightarrow\R^n$ is fixed in time. An equivalent form of the ODE is the following integral form
\begin{align}
    z(t) = z_0 + \int_{0}^{t} f(z(s)) ds.
\end{align}
The following classical result can be proved using fixed point arguments (see e.g.~\cite{Arnold1973Ordinary}, Ch. 4).
\begin{proposition}[Existence, Uniqueness and Dependence on Initial Condition]\label{prop:existence-uniqueness}
    Let $f$ be Lipschitz. Then, the solution to \eqref{eq:ode} exists and is unique. Moreover, for each $t$, $z(t)$ is a continuous function of $z_0$. If in addition, $f$ is $r$-times continuously differentiable for $r\geq 1$, then for each $t$, $z(t)$ is a $(r-1)$-times continuously differentiable function of $z_0$.
\end{proposition}

Recall that the flow map $\varphi_T:\R^n \rightarrow \R^n$ of~\eqref{eq:ode} is the mapping from $z_0 \mapsto z(T)$ where $\{z(t)\}$ satisfies~\eqref{eq:ode}. This is well-defined owing to Prop~\ref{prop:existence-uniqueness}.
Thus, we hereafter assume $f$ is Lipschitz. Let us now discuss an important constraint that such flow maps satisfy.

\begin{definition}[Orientation Preserving for Diffeomorphism Case]
    \label{def:OP}
    We call a diffeomorphism $\varphi :\R^n \to \R^n$ orientation preserving (OP) if $\det J_{\varphi}(x) > 0$ for all $x\in\R^n$, where $J_{\varphi}$ is the Jacobian of $\varphi$, i.e. $[J_{\varphi}(x)]_{ij} = \frac{\partial \varphi_i}{\partial x_j}(x)$.
\end{definition}

\begin{proposition}
    Suppose $f$ is twice continuously differentiable. Then, the flow map $\varphi_T$ of~\eqref{eq:ode} is an OP diffeomorphism.
\end{proposition}
\begin{proof}
    First, the flow map of $\dot{z} = -f(z)$ is the inverse of $\varphi_T$, and they are both differentiable due to Prop.~\ref{prop:existence-uniqueness}.

    To prove the OP property, observe that $\varphi_0$ is the identity map, and so
    \begin{align}
        \varphi_{T}(x) = \varphi_0(x) + \int_0^T f(\varphi_t(x))dt.
    \end{align}
    Thus, we have
    \begin{align}
        J_{\varphi_T}(x) = I + \int_0^T J_f(\varphi_t(x))J_{\varphi_t}(x) dt
    \end{align}
    Hence
    \begin{align}
        J_{\varphi_T}(x) = \exp
        \left(
            \int_0^T J_f(\varphi_t(x)) dt
        \right).
    \end{align}
    Since $\det \exp(A) = \exp(\tr~A) > 0$, the OP property follows.
\end{proof}

Definition~\ref{def:OP} requires differentiability. If the mapping $\varphi$ is only bi-Lipschitz, we can define the Jacobian almost everywhere, due to celebrated Rademacher's Theorem. Hence we call a bi-Lipschitz mapping OP if $\det J_{\varphi} >0 $ almost everywhere. If $\varphi$ is merely continuous, a proper definition is subtle and can be given by homological techniques, see~\cite{Hatcher:478079}.

However, if we restrict our interest to lower dimensional spaces, such as on the real line ($n=1$) or the plane ($n=2$), the definition of OP can be easily given without any differentiability requirements. In this paper, we will only need to use the $n=1$ case of OP, where a homeomorphism is OP if and only if it is increasing. The definition is a natural extension of diffeomorphism case. Below, we prove the one dimensional case that flow maps that are not necessarily differentiable must still be orientation preserving homeomorphisms. This is sufficient to prove our subsequent results. Note that this result is well-known (see~\cite{Arnold1973Ordinary}, Ch 1) and we write its proof for completeness.
\begin{proposition}\label{prop:comparison}
Let $n=1$ and $\varphi_T$ be the flow map of~\eqref{eq:ode}. Then, $\varphi_T$ is increasing.
\end{proposition}
\begin{proof}
    It is enough to show that if $z_1$ and $z_2$ are solutions of~\eqref{eq:ode}, but with different initial values $x_1 < x_2$. Then $z_1(t) < z_2(t)$ for all $t\geq 0$.
    Suppose not, we assume $z_1(t_0) =  z_2(t_0)$ for some $t_0$. Consider the following ODE:
    \begin{equation}
        \dot{w} = -f(w), \quad
        w(0) = z_1(t_0).
    \end{equation}
    Then both $z_1(t_0 - \cdot)$ and $z_2(t_0 - \cdot)$ are solutions to the above. By uniqueness we have $z_1(t_0 - t) = z_2(t_0 - t)$ for all $t$, which implies $x_1=z_1(0)=z_2(0) = x_2$, a contradiction. Since both $z_1$ and $z_2$ are continuous in $t$, we have $z_1(t) < z_2(t)$ for all $t$.
\end{proof}

Next we state a version of the well-known Gr\"{o}nwall's Inequality~\cite{gronwallNoteDerivativesRespect1919}.
\begin{proposition}[Gr\"{o}nwall's Inequality]
Let $f:\R\rightarrow \R$ be a scalar function such that $f(t)\ge 0$ and $f(t) \le A + B\int_0^t f(\tau) d\tau$. Then $f(t) \le Ae^{Bt}$.
\end{proposition}

Finally, we prove some practical results, which follow easily from classical results but are used in some proofs of the main body.
\begin{proposition}\label{prop:uniform_conv}
Let $n=1$ and $z(\cdot; x)$ be the solution of the ODE~\eqref{eq:ode} with initial value $x$. When $x$ is in some compact set $K\subset \R$, then the continuous modulus of finite time
\begin{align}
    \omega_{z(\cdot; x),[0,T]}(r) =
    \sup_{{\substack{0\le t_1\le t_2 \le T \\ |t_2 - t_1| \le r}}}
    \big|z(t_1;x) - z(t_1;x)\big|
\end{align}
converges to 0 as $r\rightarrow 0$ uniformly on $x \in K$.
\end{proposition}
\begin{proof}
    We denote $a = \min K, b = \max K$ , By Proposition~\ref{prop:comparison}, we know that $z(t;a) \le y(t;x) \le y(t;b)$, thus $H = \{ z(t;x) : x \in K, t \in [0,T] \} \subset [ \min_t z(t;a),  \max_t z(t;b)]$ is compact, so is $f(H)=\{f(h): h \in H \}$.
    Suppose $M = \max_{h \in H} |f(h)|$, we have
    \begin{align}
        \sup_{{\substack{0\le t_1\le t_2 \le T \\ |t_2 - t_1| \le r}}}
        |z(t_1;x) - z(t_2;x)| \le rM,
    \end{align}
    implying the result.
\end{proof}

The following proposition shows that in one dimension, if we have a well function, we can transport one point into another if they are located in the same side of well function's zero interval. Note that by definition, the well function cannot change sign outside of this interval.

\begin{proposition}\label{prop:1d_drive}
 Let $n=1$. Suppose $f(x) < 0$ for all $x \ge x_0$. Then for $x_0 < x_1 < x_2$. Consider the ODE:
\begin{align}
    \dot{z} = f(z), \qquad z(0) = x_2.
\end{align}
Then ultimately the ODE system will reach $x_1$, i.e., for some $T$, $z(T) = x_1$.
\end{proposition}
Before proving it, we give a simple example to illustrate this proposition. Suppose that $f(z) = -\relu(z-x_0)$, then direct computation shows $z(T) = (x_2-x_0)\exp(-T) + x_0$. In this case, $T = \ln((x_2-x_0)/(x_1-x_0))$ demonstrates the proposition. Intuitively, this proposition shows under the stated conditions, $x_0$ is an attractor for the unbounded interval $(x_0, \infty)$.
\begin{proof}
    Notice that $f$ is assumed to be continuous, and it suffices to give an estimate on $z(T)$. We only need to prove that for some $T$, $z(T) <x_1$, and the result can be easily derived by the continuity of $T\mapsto z(T)$ and that $z(0) = x_2$.
    Choose an arbitrary $\widetilde x_1 \in (x_0,x_1)$ and define $m = -\sup_{x\in[\widetilde x_1, x_2]} f(x)$.
    We have
    \begin{align}
        z(t) = x_2 + \int_{0}^{t} f(z(s)) ds.
    \end{align}
    Set $t = (x_2 - x_1) / m$. If $z(t) \leq \widetilde{x}$ (which is smaller than $x_1$) then we are done by continuity. Otherwise, we have $z(t) \le -m$ for all $t' \in [0,t]$, yielding
    \begin{align}
        z(t) \leq x_2 + \int_{0}^{t} (-m) ds = x_2 - m \frac{x_2 - x_1}{m} = x_1
    \end{align}
    which by again implies our result by continuity.
\end{proof}

With these results on ODEs in mind, we now present the proofs of our main results.

\subsection{From Approximation of Functions to Approximations of Domain Transformations}
\label{sec:prelim}

Now, we show that under mild conditions, as long as we can approximate any continuous domain transformation $\varphi: \R^n \rightarrow \R^n$ using flow maps, we can show that $\Hcalode$ is an universal approximator. Consequently, we can pass to the problem of approximating an arbitrary $\varphi$ by flow maps in establishing our main results.

\begin{proposition}\label{prop:transformation}
    Let $F: \R^n \rightarrow \R^m$ be continuous and $g:\R^n\rightarrow\R^m$ be Lipschitz. Let $K\subset \R^n$ be compact and suppose $g(\R^n) \supset F(K)$.
    Then, for any $\varepsilon>0$ and $p\in [1,\infty)$, there exists a continuous function $\varphi : K \rightarrow \R^n$ such that
    \begin{align}
        \| F - g \circ \varphi \|_{L^p(K)}
        \leq \varepsilon.
    \end{align}
\end{proposition}
\begin{proof}
    This follows from a general result on function composition proved in~\cite{li2019deepapprox}.
    We prove this in the special case here for completeness.

    The set $F(K)$ is compact, so for any $\delta>0$ we can form a partition $F(K) = \cup_{i=1,\dots,N} B_i$ with $\diam(B_i) \leq \delta$. By assumption, $g^{-1}(B_i)$ is non-empty for each $i$, so let us pick $z_i \in g^{-1}(B_i)$. For each $i$ we define $A_i = ({F}^{-1}(B_i)\cap K)$ so that $\{A_i\}$ forms a partition of $K$.
    By inner regularity of the Lebesgue measure, for any $\delta'>0$ and for each $i$ we can find a compact $K_i \subset A_i$ with $\lambda(A_i \setminus K_i) \leq \delta'$ ($\lambda$ is the Lebesgue measure) and that $K_i$'s are disjoint. By Urysohn's lemma, for each $i$ there exists a continuous function $\varphi_i : K \rightarrow [0, 1]$ such that $\varphi_i = 1$ on $K_i$ and $\varphi_i = 0$ on $\cup_{j\neq i} K_j$.

    Now, we form the continuous function
    \begin{align}
        \varphi(x) = \sum_{i=1}^{N} z_i \varphi_i(x).
    \end{align}
    We define the set $K' = \{ \sum_{i=1}^{N} \alpha_i z_i : \alpha_i \in [0,1]\}$, which is clearly compact and $\varphi(x) \in K'$ for all $x$.
    Then, we have
    \begin{align}
        \| F - g\circ \varphi \|_{L^p(K)}
        &=
        \sum_{i=1}^{N}
        \| F - g\circ \varphi \|_{L^p(K_i)}
        +
        \sum_{i=1}^{N}
        \| F - g\circ \varphi \|_{L^p(A_i \setminus K_i)}
        \nonumber \\
        &\leq
        \sum_{i=1}^{N}
        \| F - g\circ \varphi_i \|_{L^p(K_i)}
        +
        \left[
            \| F \|_{C(K)} + \| g \|_{C(K')}
        \right] N\delta'.
    \end{align}
    We take $\delta'$ small enough so that the last term is bounded by $\delta$. Then, we have
    \begin{align}
        \| F - g\circ \varphi \|_{L^p(K)}
        &\leq
        \sum_{i=1}^{N}
        \delta |K_i|
        + \delta
        \leq (1 + |K|) \delta.
    \end{align}
    Taking $\delta = \varepsilon / (1 + |K|)$ yields the result.
\end{proof}

We shall hereafter assume that $g(\R^n) \supset F(K)$, which as discussed earlier is easily satisfied. Hence we have the following immediate corollary.

\begin{corollary}\label{cor:transformation}
    Assume the conditions in Proposition~\ref{prop:transformation}. Let $\varPhi$ be some collection of continuous functions from $\R^n$ to $\R^n$ such that for any $\delta>0$ and any continuous function $\varphi_1:\R^n \rightarrow \R^n$, there exists $\varphi_2 \in \varPhi$ with
    $
        \| \varphi_1 - \varphi_2 \|_{L^p(K)} \le \delta
    $.
    Then, there exists $\varphi \in \varPhi$ such that
    $
        \| F - g\circ \varphi \|_{L^p(K)} \le \varepsilon
    $.
\end{corollary}
\begin{proof}
    Using Proposition~\ref{prop:transformation}, there is a $\varphi_1$ such that
    $\| F - g\circ\varphi_1 \|_{L^p(K)} \leq \varepsilon/2$.
    Now take $\varphi \in \varPhi$ such that $\| \varphi_1 - \varphi \|_{L^p(K)} \leq \varepsilon / (2 \lip(g))$. Then,
    \begin{align}
        \begin{split}
            \| F - g\circ \varphi \|_{L^p(K)}
            &\leq
            \| F - g\circ \varphi_1 \|_{L^p(K)}
            +
            \| g\circ \varphi_1 - g\circ \varphi \|_{L^p(K)}
            \\
            &\leq
            \frac{\varepsilon}{2} + \lip(g)
            \left(
                \int_K \| \varphi_1(x) - \varphi(x) \|^p dx
            \right)^{1/p}
            \leq \varepsilon.
        \end{split}
    \end{align}
\end{proof}

\subsection{Properties of Attainable Sets and Approximation Closures}

Owing to Corollary~\ref{cor:transformation}, for the rest of the paper we will focus on proving universal approximation of continuous transformation functions $\varphi$ from $\R^n$ to $\R^n$ by flow maps of the dynamical system
\begin{align}
    \dot{z}(t) = f_{t}(z(t)),
    \qquad
    f_t \in \Fcal,
    \qquad
    z(0) = x,
\end{align}
after which we can deduce universal approximation properties of $\Hcalode$ via Corollary~\ref{cor:transformation}.

We now establish some basic properties of flow maps and closure properties.
In principle, in our hypothesis space~\eqref{eq:hyp_dyn_noparam} we allow $t \mapsto f_t(z)$ to be any essentially bounded measurable mapping for any $z\in\R^n$. However, it turns out that to establish approximation results, it is enough to consider the smaller family of piece-wise constant in time mappings, i.e. $f_t = f_j \in \Fcal$ for $t\in [t_{j-1}, t_j)$.
For a fixed $f$ in the control family $\Fcal$, to emphasize dependence on $f$ we denote by $\varphi_{\tau}^{f}$ the flow map of the following ODE at time horizon $\tau$:
\begin{align}
    \dot{z}(t) = f(z(t)),
    \qquad
    z(0) = x.
\end{align}
That is,
\begin{align}
    \varphi_{\tau}^{f} = z(\tau)
    \qquad
    \text{where}
    \qquad
    \dot{z}(t) = f(z(t)), \quad z(0) = x, \quad t \in [0,\tau].
\end{align}
The \emph{attainable set} of a finite time horizon $T$ due to piece-wise constant in time controls, denoted as ${\Acal}_\Fcal(T)$, is defined as
\begin{equation}
    {\Acal}_\Fcal(T) = \left\{
        \varphi^{f_k}_{\tau_k}
        \circ
        \varphi^{f_{k-1}}_{\tau_{k-1}}
        \circ \cdots \circ
        \varphi^{f_1}_{\tau_1}
        ~|~
        \tau_1 + \cdots +\tau_k = T,
        f_1,\dots,f_k \in \Fcal,
        k \geq 1
    \right\},
\end{equation}
In other words, ${\Acal}_\Fcal(T)$ contains the flow map of an ODE whose right hand side is $f_i$ for $t \in [t_{i-1}, t_{i})$, $j=1,\dots,k$, with $\tau_i = t_i - t_{i-1}$ and $t_0=0$. It contains all the domain transformations that can be attained by an ODE by selecting a piece-wise constant in time driving forces from $\Fcal$ up to a terminal time $T$. The union of attainable sets over all possible terminal times, ${\Acal}_\Fcal = \cup_{T>0} {\Acal}_\Fcal(T)$, is the overall attainable set. In view of Corollary~\ref{cor:transformation}, to establish the approximation property of $\Hcalode$ it is sufficient to prove that any continuous transformation $\varphi$ can be approximated by mappings in ${\Acal}_\Fcal$. Note that ${\Acal}_\Fcal(T)$ is a subset of $\varPhi(\Fcal, T)$ defined in~\eqref{eq:varPhi_defn}.

Now, let us prove some properties of the approximation closure (i.e. closure with respect to the topology of compact convergence, see Sec.~\ref{sec:notation}) of attainable sets.

\begin{lemma}\label{lem:1d_increasing}
Let $n=1$. If $A$ is a family of continuous and increasing functions from $\R$ to $\R$, then $\overline{A}$ contains only increasing functions.
\end{lemma}
\begin{proof}
	Immediate.
\end{proof}

Next, we state and prove an important property about approximation closures of control families:
$\mathcal F$ shares the same approximation ability as $\chbar({\mathcal F})$ when used to drive dynamical systems. However, a convex hull of Lipschitz function family might not be a Lipschitz function family in general. Hence we adopt a slightly different description.
\begin{restatable}{proposition}{convexhull}
\label{prop:convexhull}
	Let $\Fcal$ be a Lipschitz control family. Then, for any Lipschitz control family $\widetilde{\Fcal}$ such that $\mathcal F \subset \widetilde{\Fcal} \subset \chbar({\mathcal F})$, we have
	\begin{align}
		\overline{{\Acal}_{\mathcal F}} = \overline{\Acal_{\widetilde{\Fcal}}}.
	\end{align}
\end{restatable}
Proposition~\ref{prop:convexhull} is an important result concerning the effect of continuous evolution, which can be regarded as a continuous family of compositions: any function family driving a dynamical system is as good as its convex hull in driving the system, which can be an immensely larger family of functions. Similar properties of flows have been observed in the context of variational problems, see \cite{Warga1962Relaxed}. This is a first hint at the power of composition on function approximation.

To prove Proposition~\ref{prop:convexhull} we need the following lemmas.
\begin{lemma}
	\label{lemma:ACAC}
	If ${\Acal}_{\mathcal F}$, ${\Acal}_{\widetilde{\Fcal}}$ are attainable sets of $\Fcal$, $\widetilde{\Fcal}$ and $\Fcal \subset \widetilde{\Fcal} \subset \overline{\Fcal}$
	. Then we have
	\begin{align}
		\overline{{\Acal}_{\F}} = \overline{\Acal_{\widetilde{\Fcal}}}.
	\end{align}
\end{lemma}
\begin{proof}
	It suffices to show that ${\Acal}_{\widetilde{\Fcal}} \subset \overline{{\Acal}_{\F}}$, which implies $\overline{{\Acal}_{\F}} \subset \overline{\Acal_{\widetilde{\Fcal}}} \subset \overline{{\Acal}_{\F}}$ and hence the lemma.
	Note that any $\tilde \varphi \in {\Acal}_{\widetilde{\Fcal}}$ is of the form
	\begin{align}
		\tilde \varphi =
		\varphi^{\tilde f_k}_{t_k}
		\circ
		\varphi^{\tilde f_{k-1}}_{t_{k-1}}
		\circ
		\cdots
		\circ
		\varphi^{\tilde f_{1}}_{t_{1}}
	\end{align}
	where each $\tilde f_{i} \in \widetilde{\Fcal}$.
	To prove the lemma, we have to show that for any compact $K\subset \R^n$ and any $\varepsilon>0$, we can construct a function $\varphi \in {\Acal}_\Fcal$ such that $\|\tilde \varphi - \varphi\|_{C(K)} \le \varepsilon$.
	We prove this by induction on $k\geq 0$. First, the case when $k=0$ is obvious since it is just the identity mapping.
	Suppose now that the statement holds for $k-1$ and fix any compact $K$. Write $\tilde \varphi = \varphi^{\tilde f_k}_{t_k} \circ \tilde \psi$, where $\tilde \psi $ is composition of $k-1$ flow maps driven by $\widetilde{\Fcal}$. By the inductive hypothesis, for any $\varepsilon_1$ (to set later) there exists $\psi \in {\Acal}_\Fcal$, such that $\|\psi - \tilde \psi\|_{C(K)} \le \varepsilon_1$. Moreover, by the assumption $\widetilde{\Fcal} \subset \overline{\Fcal}$, for any $\varepsilon_2$ and compact $K'$ there is a function $f_k\in \Fcal$ such that $\|f_k - \tilde f_k\|_{C(K')} \le \varepsilon_2$. Here, we choose $K' = \{ x ~|~ \inf_{y\in \psi(K)} \|x - y\| \leq  2(\varepsilon_1 + t_k)e^{t_k\lip(\tilde f_k)}) \}$ and $\varepsilon_2 < 1$.

	Now, consider two ODEs:
	\begin{equation}
	\tilde z(t) = \tilde \psi(x) + \int_{0}^{t} \tilde f_k(\tilde z(\tau) )d\tau,
	\end{equation}
	and
	\begin{equation}
	z(t) = \psi(x) + \int_{0}^{t} f_k(z(\tau)) d\tau.
	\end{equation}
	We have $\tilde{\varphi}(x) = \tilde{z}(t_k)$ and $x\mapsto z(t_k)$ belongs to ${\Acal}_\Fcal$, thus it remains to show that the solutions of these ODEs can be made arbitrarily close. In fact, we show the following estimates holds:
	\begin{equation}
		\label{eq:ac-estimates}
		\sup_{0\leq t \leq t_k}
		 |z(t) - \tilde z(t)| <
		 2(\varepsilon_1 + t_k\varepsilon_2)
		 e^{t_k\lip(\tilde f_k)}.
	\end{equation}
	We prove by contradiction. Suppose not, then set
	\begin{align}
		t =
		\inf
		\{
			s \geq 0
			~:~
			|z(s) - \tilde z(s)| \ge 2(\varepsilon_1 + t_k\varepsilon_2)e^{t_k\lip(\tilde f_k)}
		\}
	\end{align}
	By continuity we have $|z(t) - \tilde z(t)| \ge 2(\varepsilon_1 + t_k\varepsilon_2)e^{t_k\lip(\tilde f_k)}$. By assumption we also have $|\tilde f_k(z(\tau)) - f_k(z(\tau))|\le \varepsilon_2$ for all $\tau < t$.
	By subtracting, we have
	\begin{equation}
	\begin{split}
	|\tilde z(t) - z(t)| \le& \varepsilon_1 + |\int_{0}^{t} \tilde f_k (\tilde z(\tau)) d\tau - \int_{0}^{t} f_k(z(\tau)) d\tau| \\
	\le& \varepsilon_1
	+ \int_0^t |\tilde f_k( z(\tau)) - f_k(z(\tau))| d\tau
	+ \int_{0}^t|\tilde f_k(\tilde z(\tau)) - \tilde f_k( z(\tau))| d\tau \\
	\le& \varepsilon_1 + \varepsilon_2t + \lip({\tilde f_k}) \int_0^t |\tilde z(\tau) -z(\tau)| d\tau.
	\end{split}
	\end{equation}
	By Gr\"{o}nwall's inequality, we have
	\begin{equation}
	\label{eq:Gronwall}
	\sup_{0\leq t \leq t_k}
	|\tilde z(t) - z(t)| \le (\varepsilon_1 + t_k\varepsilon_2) e^{t_k\lip({{\tilde f}_k})},
	\end{equation}
	which contradicts to the choice of $t$.  Hence \eqref{eq:ac-estimates} holds, and we can choose $\varepsilon_1$ and $\varepsilon_2$ arbitrarily small, which concludes the proof.
\end{proof}

\begin{lemma}
	\label{lemma:add}
Suppose $f,g \in \mathcal F$ and $t>0$, then we have
$\varphi^{(f+g)/2}_{t} \in \overline{{\Acal}_{\mathcal F}}$.
\end{lemma}

\begin{proof}
    We will show that $\varphi^f_{t/2N} \circ \varphi^g_{t/2N}\circ \cdots \circ \varphi^f_{t/2N} \circ \varphi^g_{t/2N}$ can approximate $\varphi^{(f+g)/2}_t \in \overline{{\Acal}_\Fcal}$ arbitrarily well by increasing $N$. The mapping $\varphi^{(f+g)/2}_t$ is the solution of
    \begin{equation}
    \begin{split}
		z(t) =& x + \int_0^t \left(\frac{f+g}{2}\right)
		(z(\tau))d\tau \\
         =& x + \int_{0}^{t/2N} + \int_{2t/2N}^{3t/2N} + \cdots + \int_{(2N-2)t/2N}^{(2N-1)t/2N} (f+g)(z(\tau)) d\tau \\
         &+ \int_{t/2N}^{2t/2N} + \cdots + \int_{(2N-1)t/2Nt}^t \left[
			 \left(\frac{f+g}{2}\right)(z(\tau)) - \left(\frac{f+g}{2}\right)
			 \left(z\left(\tau - \frac{t}{2N}\right)\right)
		\right] d\tau\\
		=& x + \int_{0}^{t/2N} + \int_{2t/2N}^{3t/2N} + \cdots + \int_{(2N-2)t/2N}^{(2N-1)t/2N} f(z(\tau))d\tau \\
		&+ \int_{t/2N}^{2t/2N} + \cdots + \int_{(2N-1)t/2Nt}^t g(z(\tau))d\tau \\ &+ \int_{t/2N}^{2t/2N} + \cdots + \int_{(2N-1)t/2Nt}^t
		\left[
			\left(\frac{f+g}{2}\right)
			(z(\tau)) - \left(\frac{f+g}{2}\right)
			\left(z\left(\tau - \frac{t}{2N}\right)\right)
		\right]  \\
		&
		\qquad\qquad\qquad\qquad
		+
		\left[
			g\left(
				z
				\left(
					\tau-\frac{t}{2N}
				\right)
			\right) -g(z(\tau))
		\right] d\tau.
    \end{split}
    \end{equation}
    Thus if $w(t)$ satisfying:
    \begin{equation}
    \begin{split}
    w(t) =& x + \int_{0}^{t/2N} + \int_{2t/2N}^{3t/2N} + \cdots + \int_{(2N-2)t/2N}^{(2N-1)t/2N} f(w(\tau))d\tau +\\ &+ \int_{t/2N}^{2t/2N} + \cdots + \int_{(2N-1)t/2Nt}^t g(w(\tau))d\tau.
    \end{split}
    \end{equation}
    Then we have
    \begin{equation}
    \begin{split}
		|z(t)-w(t)| \le& \int_0^t\max(\lip(f), \lip(g))|z(\tau)-w(\tau)|d\tau \\&+ \frac{t}{2}\omega_{z,[0,t]}
		\left(\frac{t}{2N}\right)
		\left[
			\lip
			\left(\frac{f+g}{2}\right)+\lip(g)
		\right].
    \end{split}
	\end{equation}
	Recall that $\omega$ is the modulus of continuity defined in Proposition~\ref{prop:uniform_conv}.
	Again, by Gr\"{o}nwall's inequality we have
	\begin{align}
		|z(t)-w(t)|\le \frac{t}{2}\omega_{z,[0,t]}
		\left(
			\frac{t}{2N}
		\right)
		\left[
			\lip
			\left(
				\frac{f+g}{2}
			\right)
			+
			\lip(g)
		\right]
		e^{\max(\lip(f), \lip(g))}.
	\end{align}
    For any selected compact set $K$, $\omega_{z,[0,t]}(\frac{t}{2n}) ~\to ~0$  by Proposition~\ref{prop:uniform_conv}, thus we obtain $\varphi^{(f+g)/2}_t \in \overline{{\Acal}_{\mathcal F}}$.
\end{proof}

Now, we are ready to prove Proposition~\ref{prop:convexhull}.
\begin{proof}[Proof of Proposition~\ref{prop:convexhull}]
	Using the same technique in the proof of Lemma \ref{lemma:add}, we can show that for $f_1, \cdots, f_m \in \mathcal F$, $\varphi^h_t \in \overline{{\Acal}_{\mathcal F}}$, where $h = \sum_i q_i f_i$ for some rational numbers $q_i$.  Let ${\Acal}_{\Fcal'}$ be the attainable set with control family $\mathcal F'  = \{\sum_{i=1}^{m} q_i f_i: q_i \in \mathbb Q, \sum_i q_i = 1,f_i \in \F, m\in\N \}$,  then we have $\overline{{\Acal}_{\Fcal'}} = \overline{{\Acal}_{\F}}$. Since $\overline{\Fcal'} = \overline{\Fcal}$, we arrive at the desired result using Lemma~\ref{lemma:ACAC}.
\end{proof}
\section{Proof of Main Results}

In this section, we prove the main results (Theorem~\ref{thm:main} and~\ref{thm:1D}). We start with the one dimensional case to gain some insights on how a result can be established in general, and in particular, elucidate the role of well functions (Definition~\ref{def:well_function}) in constructing rearrangement dynamics. This serves to motivate the extension of the results in higher dimensions.

\subsection{Approximation Results in One Dimension and the Proof of Theorem~\ref{thm:1D}}
\label{sec:res_1d}

We take $n=1$ in this subsection.
Proposition~\ref{prop:comparison}, together with the fact that compositions of continuous and increasing functions are again continuous and increasing, implies that any function from ${\Acal}_{\Fcal}$ must be continuous and increasing. This poses a restriction on the approximation power of $\overline{{\Acal}_{\Fcal}}$ as the following result shows:
\begin{proposition}\label{prop:negative-result}
    Let $n=1$ and $\mathcal F$ be a Lipschitz control family, whose attainable set is ${\Acal}_{\Fcal}$. Then $\overline{{\Acal}_{\Fcal}}$ contains only increasing functions.
\end{proposition}
\begin{proof}
    Proposition~\ref{prop:comparison} implies that any function in ${\Acal}_{\Fcal}$ is continuous and increasing, since both properties are closed under composition. The proposition then follows from Lemma~\ref{lem:1d_increasing}.
\end{proof}

It follows from Proposition~\ref{prop:negative-result} that any continuous function $\varphi$ that is strictly decreasing over an interval $[c,d]$ cannot be approximated by $\overline{{\Acal}_{\Fcal}}$. Nevertheless, it makes sense to ask for the next best property: can $\overline{{\Acal}_{\Fcal}}$ approximate any continuous and increasing function?

To investigate this problem, we first select an appropriate control family, which corresponds to deep neural networks with ReLU activations, and see if it can indeed approximate any such function. We will remove this explicit architectural assumption later. The ReLU control family in $n=1$ is given by
\begin{equation}\label{def:relu-control}
    \F =
    \left\{
        v \relu(w\cdot + b): v,w,b \in \R
    \right\}.
\end{equation}
Notice that the ReLU control family~\eqref{def:relu-control} satisfies the restricted affine invariant condition as defined in Definition~\ref{def:res_affine_inv}.

We now show that in one dimension, flow maps of ODEs driven by the ReLU control family can in fact approximate any continuous function.
\begin{proposition}
    \label{prop:relu}
    Let $\varphi : \R \rightarrow \R$ be continuous and increasing and $\Fcal$ be the 1D ReLU control family~\eqref{def:relu-control}. Then, for any $\varepsilon > 0$ and compact $K\subset \R$, there exists $\widehat{\varphi} \in {\Acal}_\Fcal$ such that $\| \varphi - \widehat{\varphi} \|_{C(K)} \leq \varepsilon$. In other words, $\varphi \in \overline{{\Acal}_\Fcal}$.
\end{proposition}

\begin{proof}
    We need the following lemma, from which we can deduce the desired result .
	\begin{lemma}
		\label{lemma:one-dim}
		Let $M\geq 1$. Given $x_1<\cdots <x_M$ and $y_1<\cdots<y_M$, there exists a function $\psi \in {\Acal}_\Fcal$ such that $\psi(x_i) = y_i$.
	\end{lemma}
    We postpone the proof of Lemma~\ref{lemma:one-dim} and first show how to prove Proposition~\ref{prop:relu} from it. By replacing $K$ with a larger set, we can always assume that $K$ is a closed interval. Consider a partition $\Delta$ on $K$, with nodes $x_1<\cdots<x_M$.

    By Lemma~\ref{lemma:one-dim}, we can find $\psi \in {\Acal}_\Fcal$ such that $\psi(x_i) = \varphi(x_i)$ for all $i=1,\dots,M$. Therefore
    \begin{equation}
	\psi(x) - \varphi(x) \le \psi(x_{i+1}) - \varphi(x_i) \le \varphi(x_{i+1}) - \varphi(x_i) \le \omega_{\varphi}(|\Delta|)
	\end{equation}
	whenever $x\in [x_i, x_{i+1}]$. Here $|\Delta|:=\max_{1\le i \le M} |x_{i} - x_{i-1}|$. We deduce that $\psi(x) - \varphi(x) \geq - \omega_{\varphi}(|\Delta|)$  holds for the same reason. Hence we have $\|\varphi - \psi\|_{C(K)} \le \omega_{\varphi}(|\Delta|)$. Since $\varphi$ is continuous, sending $|\Delta|$ to 0 and using Proposition~\ref{prop:uniform_conv} gives the desired result.
\end{proof}

Now it remains to prove Lemma \ref{lemma:one-dim} constructively.
To do this, first observe that the definition of well function (Definition~\ref{def:well_function}) when specialized to one dimension is a function $h_Q$ such that $h_Q(x)=0$ if and only if $x \in Q=[q_1,q_2]$ for some $q_2>q_1$. This can be constructed by the ReLU family (c.f. Example~\ref{exa:relu}) by
\begin{align}
    h_{Q} = \frac{1}{2} [\relu(q_1-x) + \relu(x-q_2)].
\end{align}
Obviously, $h_Q \in \ch(\Fcal) \subset \chbar(\Fcal)$, so that the condition that the latter contains a well function is trivially satisfied for the ReLU control family.

\begin{proof}[Proof of Lemma \ref{lemma:one-dim}]
    By Proposition \ref{prop:convexhull}, we denote $\widetilde{\Fcal} = \mathcal F \cup \{ h_Q: Q \subset K \}$. We will show that $\widetilde{\Fcal}$ can produce the desired approximation property.
	We construct using induction a mapping $\varphi_k$ which maps $x_i$ to $y_i$ for $i=1,2,\cdots,k$. First we show the base case $k=1$. Take $h_Q$ to be the well function with respect to $Q = [q_1,q_2]$. Since $\F$ is translation invariant, we can suppose that both $x_1$ and $y_1$ are greater than $q_2$. Since $h_Q$ does not change sign in $[q_2, \infty)$, by Proposition~\ref{prop:1d_drive} we know that either $\varphi^{h_Q}_t$ or $\varphi^{-h_Q}_t$ can map $x_1$ into $y_1$ for some $t$. Thus we prove the base case since $\F$ is symmetric.

	Suppose we have $\varphi_k$, now we will construct $\varphi_{k+1}$ based on $\varphi_k$. Applying $\varphi_k$, we may assume that $x_i = y_i, i=1,2,\cdots,k$. Again we assume $h_Q$ is a well function with zero interval $Q = [q_1,q_2] (q_2<\min(x_1,y_1))$ and $h_{Q'}$ is a well function with interval $Q'=[q_0,q_1]$. We further assume that $h_{Q'}(x) <0$ on $[q_1, \infty)	$, $h_Q(x) <0$ on $[q_2, \infty)$, otherwise we can use $-h_Q$ or $-h_{Q'}$ in their places.

    Let $t_1 = \inf\{t: \varphi^{h_{Q'}}_t(x_k) < q_2\}$ and $t_2 = \sup\{t:\varphi^{h_{Q'}}_t(x_{k+1})>q_2, \varphi^{h_{Q'}}_t(y_{k+1})>q_2\}$. Clearly we have $t_1<t_2$. Choose any $t' \in (t_1,t_2)$, and $\psi = \varphi^{\pm _{h_{Q}}}_t$ mapping $\varphi^{h_{Q'}}_{t'}(x_{k+1})$ to $\varphi^{h_{Q'}}_{t'}(y_{k+1})$, we construct
    \begin{align}
        \varphi_{k+1} = \varphi^{-h_{Q'}}_{t'} \circ \psi \circ \varphi^{h_{Q'}}_{t'} \circ \varphi_k
    \end{align}
    as desired. By induction, we have completed the proof of Lemma \ref{lemma:one-dim}.
\end{proof}

\paragraph{Sufficient Conditions for Approximation of Continuous and Increasing functions and the Proof of Theorem~\ref{thm:1D}.}

We showed previously that all continuous and increasing functions can be approximated by ReLU-driven dynamical systems. In this section, we shall do away with an explicit architecture, which leads to the proof of Theorem~\ref{thm:1D}. The key observation from the proof of Lemma~\ref{lemma:one-dim}, is that all we really need is having a well function contained in $\chbar({\Acal}_{\Fcal})$ that we can translate and change signs, which is achieved by a restricted affine invariance assumption. On the other hand, whether or not $\Fcal$ itself is a ReLU control family, or any other specific family, is inconsequential. This motivates us to ask the question of sufficiency: what assumptions on $\Fcal$ is enough to guarantee that it is a universal control family? Notice that instead of constructing an explicit well function in the form of the average of two ReLU functions, we can just use an arbitrary well function as defined in~\ref{def:well_function} to drive the dynamics. The following result makes this precise.

\begin{proposition}
\label{prop:ua_sufficiency}
Assume the control family $\mathcal F$ is symmetric and translation invariant, which is equivalent to restricted affine invariant with $D=\pm 1$ and $A=1$ in Definition~\ref{def:res_affine_inv}, and that $\chbar (\Fcal)$ contains a well function. Then, the conclusion in Proposition~\ref{prop:relu} holds.
\end{proposition}
\begin{proof}
    The proof is almost identical to that of Proposition~\ref{prop:relu} with the well function constructed by averaging two ReLU functions replaced by a general well function contained in $\chbar{(\Fcal)}$.
    Notice that since a well function does not change sign out of $I$, by choosing a proper sign one can always shrink a finite point arbitrarily close to the interval. This follows from Proposition~\ref{prop:1d_drive}.
\end{proof}

Proposition~\ref{prop:ua_sufficiency} combined with Corollary~\ref{cor:transformation} implies Theorem~\ref{thm:1D}. Clearly, Proposition~\ref{prop:ua_sufficiency} generalizes Proposition~\ref{prop:relu}. It also follows that if $\chbar (\Fcal) $ contains all continuous functions, it must contain in particular a well function and so ${\Acal}_\Fcal$ has the desired approximation property. However, this is not necessary for universal approximation to hold.

\begin{remark}
    In one dimension, the ability for a dynamical system to approximate any continuous and increasing function has the immediate consequence that if we were to embed the dynamical system in two dimensions, then we can approximate any continuous function $\varphi$ of bounded variation, as long as we are allowed a linear transformation in the end, e.g. if $g$ in Prop.~\ref{prop:transformation} is linear. This is because a continuous function of bounded variation can always be written as a difference of two continuous and increasing functions. However, this does require embedding in high dimensions. We will show later that for $n\geq 2$, embedding is not necessary to achieve universal approximation.
\end{remark}

\subsection{Approximation Rates in One Dimension}
\label{sec:rates}

All results so far are on whether a given function can be approximated by a dynamical system with control families satisfying certain conditions. However, by the very definition of the attainable set we are forced to consider dynamical systems of finite, but arbitrarily large time horizons. Just like in the development of traditional approximation theory, one may be interested to ask the following: given an approximation budget, how well can we approximate a given function? Perhaps a more pertinent question is this: what kind of functions can be efficiently approximated by dynamical systems? There are more than one way to define the notion of budget. Here, we will consider a natural one in continuous time: the time horizon $T$.

In this part, we give some results in this direction in the simplest case: the one dimensional case ($n=1$) and the ReLU activation control family. For convenience of exposition, we assume that our target function $\varphi$ is defined on $[0,1]$. We postpone results on general control families in higher dimensions to future work.

To properly quantize the efficiency, we should eliminate the positive homogeneity of the ReLU control function, which masks the effect of the time horizon $T$ due to the ability to arbitrarily rescale time.  Therefore, we restrict $|v|,|w|\le 1$ in $v \relu(w\cdot  + b)$ and then the quantity of time horizon becomes meaningful.

\begin{remark}
    An alternative is using $\int |w| dt$ to measure the approximation cost in place of $T$. This notation is related to the Barron space analysis~\cite{ma2019barron}. It can be checked that one can change $T$ into $\int |w| dt$ in the following results. Lastly, it is also possible to measure the complexity of the variation of $w,v,b$ in time. Here, we do not consider such cases.
\end{remark}

First, we show in the following lemma that if $\varphi$ is piecewise linear, then it can be represented by functions in ${\Acal}_{\Fcal}(T)$ for some $T$ large enough.

Before introducing the following lemma, we first define the Total Variation with a slight modification. Suppose $u$ is a function defined on $[0,1]$, we extend $u$ to $u_E$ such that $u_E = 0$ in $[-\varepsilon, 0) \cup (1, 1+\varepsilon]$. We define $\|u\|_{\TV} = \|u_E\|_{\TV[-\varepsilon, 1+\varepsilon]}$, the latter is defined as
\begin{align}
    \|f\|_{\TV[-\varepsilon, 1+\varepsilon]} = \sup_{-\varepsilon = x_0<\cdots < x_M = 1+ \varepsilon} \sum_{i=1}^{M} |f(x_i) - f(x_{i-1})|
\end{align}
\begin{lemma}\label{lemma:heaviside function}
    If $\ln \varphi'$ is a piecewise constant function with $N$ pieces, then $\varphi$ can be written as
    \begin{align}
        \varphi = \varphi^{g_1}_{t_1}\circ \cdots \circ \varphi^{g_{N-1}}_{t_{N-1}} + c.
    \end{align}
    Here, $g_i \in \Fcal$ and $\varphi^{\cdot}_{\cdot}$ denotes the flow maps as defined in Section~\ref{sec:prelim} and $c$ is a constant.
    Moreover, we have $\varphi \in {\Acal}_{\Fcal}(T)$ for $T \ge \|\ln \varphi'\|_{\TV}$
\end{lemma}
\begin{proof}
We first show an auxiliary result, suppose that \begin{align}
    \varphi = \varphi^{g_1}_{t_1}\circ \cdots \circ \varphi^{g_{N-1}}_{t_{N-1}} + c,
\end{align}
Then $\ln \varphi'$ can be written as the sum of $N-1$ Heaviside functions.
The proof is by induction. For the $N=1$ case it can be checked by direct calculation, and suppose that $\varphi = \varphi^{g_1}_{t_1} \circ \varphi_2$ and we have
\begin{align}
    \varphi = {\varphi^{g_1}_{t_1}}'(\varphi_2(x))\varphi_2'(x)
\end{align}
by chain rule. Thus
\begin{align}
    \ln \varphi' = \ln {\varphi^{g_1}_{t_1}}'(\varphi_2(x)) + \ln \varphi_2'(x)
\end{align}
and the first term in RHS is a Heaviside function, while the second term is a sum of $N-2$ Heaviside functions by induction hypothesis. Hence we prove the result by induction.

The proof of the original proposition is inductive by construction. Take derivative on
$
    \varphi = \varphi^{g_1}_{t_1}\circ \cdots \circ \varphi^{g_{N-1}}_{t_{N-1}} + c,
$
we obtain $\ln \varphi' = \ln {\varphi^{g_1}_{t_1}}'+ \cdots + \ln {\varphi^{g_{N-1}}_{t_{N-1}}}'$. Since $\ln \varphi'$ is a piecewise constant function, it can be written as
\begin{align}
    \ln \varphi' = H_1 + H_2 + \cdots + H_{N-1},
\end{align}
where all $H_{i}$ are Heaviside functions. Integrate ${\varphi^{g_{N-1}}_{t_{N-1}}}'(x) = H_{N-1}(x)$ we obtain $\varphi^{g_{N-1}}_{t_{N-1}}$, and then integrate
\begin{align}
    {\varphi^{g_{N-2}}_{t_{N-2}}}'(x) = H_{N-2}(\varphi^{g_{N-1}}_{t_{N-1}})^{-1}(x)
\end{align}
we obtain $\varphi^{g_{N-2}}_{t_{N-2}}$, and so on. One can easily verify that each $\varphi^{\cdot}_{\cdot}$ is a flow map generated by some ReLU activation function. Hence we prove the first part of the proposition.

For the second part, we notice that if $f = \relu(wx+b)$, then $\ln {\varphi^{f}_t}'$ is a Heaviside function with a jump at $x = -b/w$. Since we can find a decomposition $\sum_i H_i$ such that $\sum_i |w_i| = \|\ln \varphi'\|_{\TV}$. Thus the second part of proposition is proven.
\end{proof}

From the proof of the previous lemma, we know that $T = \|\ln \varphi'\|_{\TV[0,1]}$ is optimal, since $\sum_i |a_i| \le \|\ln \varphi'\|_{\TV[0,1]}$ holds (TV is a semi-norm).
In view of this lemma, we prove the following, which gives a quantitative approximation result.
\begin{proposition}\label{prop:quantitative}
    Suppose $\varphi:[0,1]\to \mathbb R$ is an increasing function. Moreover, suppose $\varphi$ is piecewise smooth and $T_0 := \|\ln \varphi'\|_{\TV([0,1])} \le \infty$. Then
    $\varphi  \in \overline{{\Acal}_{\Fcal}(T)}$ for $T\ge T_0$.
\end{proposition}
\begin{proof}
The key ideas of the proof is separated into two parts. The first part is to show that the constant in Lemma \ref{lemma:heaviside function} has negligible cost, by considering $\varphi^{\varepsilon \relu(\cdot + M)}_{t}\varphi^{-\varepsilon \relu(\cdot)}_{t}$.
This provides a translation on $[0,1]$: $x\mapsto x + (e^{\varepsilon}-1)M$. By sending $M \to \infty$ we can construct any translation with negligible time cost.

The second part is that if $|\ln \varphi'(x) - \ln \psi'(x)| \le \varepsilon$, and $\varphi(0) = \psi(0)$, then $|\varphi(x)-\psi(x)| \le (e^{\varepsilon}-1)\|\varphi'\|_{C[0,1]}$.

Now it suffices to prove a rather simple result: given a function $u = \ln \varphi'$,  on each pieces $I'$ of $u$, we can use piecewise constant function $v|_{I'}$ to approximate $u|_{I'}$ (restriction on $I'$) such that $\|v\|_{\TV[0,1]} \le \|u\|_{\TV[0,1]}$ and  $\|v-u\|_{\infty} \le \varepsilon$. Thus we can find a function $\varphi \in {\Acal}_{\Fcal}((1-\varepsilon)T_0)$ such that $\ln \varphi' = (1-\varepsilon)v$. By compositing a translation, we know that there exists $\psi \in {\Acal}_{\Fcal}(T_0)$ such that $\|\psi - \varphi \|_{C[0,1]} \le \exp(\varepsilon\|\varphi \|_{C[0,1]} + \varepsilon)$.
Thus we conclude that $\varphi \in \overline{{\Acal}_{\Fcal}(T_0)}$.
\end{proof}

The preceding results show that it is possible to constrain the approximation space to flow maps with time horizon up to some finite $T_0$, provided that the target function $\varphi$ is such that $\| \ln \varphi' \| \leq T$. In this sense, the total variation of the logarithm of $\varphi$ is a measure of complexity under our compositional approximation procedure.

Let us now develop the quantitative results a little further for the case where $T$ is not sufficiently large, i.e. $T < \|\ln \varphi'\|_{\TV}$. This involves analyzing the error
\begin{equation}
    E_T(\varphi) = \inf_{\psi \in {\Acal}_{\Fcal}(T)} \|\varphi - \psi\|_{C(K)},
\end{equation}
which may be non-zero when $T < \|\ln \varphi'\|_{\TV}$.

\begin{proposition}
$E_T(\varphi)$ is given by the following optimization problem
\begin{equation}
\label{eq:exact}
    \inf_{\psi} ~~ \|\varphi - \psi\|_{C[0,1]}, \qquad s.t.~~ \|\ln \psi'\|_{\TV} \leq T.
\end{equation}
\end{proposition}
Notice that the existence of $\ln \psi'$ implies that $\psi$ is a continuously increasing function, so we may consider the above optimization problem only for the case where $\psi$ is continuously increasing.

It is generally hard to work with optimization problems such as~\eqref{eq:exact}, since it involves total variations of logarithms of functions. Below, we formulate its relaxed version.
\begin{proposition}
Denote the relaxed optimization problem
\begin{equation}
\label{eq:relaxed}
  \gamma(u,T) =  \inf_v ~~ \|u - v\|_{C[0,1]}, \qquad s.t. ~~\|v\|_{\TV} \le T.
\end{equation}
Then $E_T(\varphi) \le [\exp(\gamma(\ln \varphi',T)) - 1]\|\varphi'\|_{C[0,1]}$
\end{proposition}
\begin{proof}
We choose $v$ such that $\|v\|_{\TV} \le T$ and
\begin{align}
    \|\ln \varphi' - v\|_{C[0,1]} \le \gamma(\ln \varphi', T) + \varepsilon.
\end{align}

Choose $\psi$ such that $\ln \psi' = v$ and $\psi(0) = \varphi(0)$. Then since $\|\ln\frac{\psi'}{\varphi'}\|_{\TV} \le \gamma(\ln \varphi', T) + \varepsilon$, we have $|\frac{\psi'}{\varphi'} - 1|\le \exp(\gamma(\ln \varphi', T) + \varepsilon) - 1$.

Hence
\begin{align}
    |\varphi(x) - \psi(x)|
    \le \int_{0}^x
    |\varphi'(x)|
    \left| 1-\frac{\psi'(x)}{\varphi'(x)}\right|
    \le \|\varphi'\|_{C[0,1]} \left[\exp(\gamma(\ln \varphi', T)  \varepsilon) - 1\right].
\end{align}
Sending $\varepsilon \to 0$, we arrive at the result.
\end{proof}

In general, both \eqref{eq:exact} and \eqref{eq:relaxed} are hard to solve. However, for some simple cases of $u$, the problem~\eqref{eq:relaxed} has explicit solution. For example, if $u$ itself is a increasing function, then the solution of \eqref{eq:relaxed} is $\frac{1}{2}(\|u\|_{\TV} - T)$. If $u$ is increasing in $[0,s]$ and decreasing in $[s,1]$, then the solution of \eqref{eq:relaxed} is $\frac{1}{4}(\|u\|_{\TV} - T)$. This gives approximation rates for specific cases, but a general investigation of these approximation rates are postponed to future work.
\subsection{Approximation Results in Higher Dimensions and the Proof of Theorem~\ref{thm:main}}
\label{sec:res_nd}

In this section, we will generalize the universal approximation results to higher dimensions. The interesting finding is that in higher dimensions, the fact that ${\Acal}_{\Fcal}$ contains only OP homeomorphisms no longer poses a restriction on approximations in the $L^p$ sense. Moreover, the sufficient condition for universal approximation in higher dimensions is closely related to that in one dimension, where the rearrangement dynamics are driven by well functions. We will prove the following result, which together with Corollary~\ref{cor:transformation} implies Theorem~\ref{thm:main}.

\begin{proposition}
\label{prop:high-dim-result}
    Let $n\geq 2$.
    Suppose $\Fcal$ is restricted affine invariant and $\chbar(\Fcal)$ contains a well function. Then for any compact set $K$, $p\in [1,\infty)$,  $\varphi \in $ $C(\R^n)$ , $\varepsilon>0$, there exists a mapping $\tilde \varphi \in {\Acal}_{\mathcal F}$, such that $\|\tilde \varphi - \varphi\|_{L^p(K)} \le \varepsilon.$
\end{proposition}
We notice that for the purpose of approximation, the fact $\chbar(\Fcal)$ contains a well function $h$ allows us to assume without loss of generality that $\Fcal$ contains a well function.
To see this, denote by $\widetilde{\Fcal}$ the smallest restricted affine invariant set containing $\Fcal \cup {h}$.
We have $\Fcal \subset \widetilde{{\Fcal}} \subset \chbar(\Fcal)$.
Proposition~\ref{prop:convexhull} then says that $\overline{{\Acal}_{\mathcal F}}  = \overline{\Acal_{\widetilde{\Fcal}}}$, hence we can prove approximation results using $\widetilde{\Fcal}$ in place of $\Fcal$.

\subsubsection{Preliminaries}

    In order to prove Proposition~\ref{prop:high-dim-result}, we require a few preliminary results which we state and prove in this subsection.
    The key approach in proving the proposition is similar to the one dimensional case: we show that we can transform a finite number of distinct source points into a finite number of target points, which are not necessary distinct. More precisely, we show the following proposition, which generalize Lemma~\ref{lemma:one-dim} given in the one dimensional case.

	\begin{lemma}
		\label{lemma:high-dim}
		Suppose $\Fcal$ contains a well function. Let $\varepsilon>0$ and $x^1, \cdots, x^m, y^1, \cdots, y^m \in \R^n$ be such that $\{x^k\}$ are distinct points. Then there exists $\psi \in {\Acal}_{\mathcal F}$ such that $|\psi(x^k)-y^k| \le \varepsilon$ for all $k = 1,\dots,m$.
	\end{lemma}

	Lemma~\ref{lemma:high-dim} follows from the combination of the following two lemmas.

	\begin{lemma}
	\label{lemma:perturbation-argument}
	Suppose $\Fcal$ contains a well function and $x^1,\cdots, x^m$ are distinct points. Then given any $\varepsilon >0$, there exists a flow map $\psi \in {\Acal}_{\Fcal}$ such that $|\psi(x^k) - x^k| \le \varepsilon$, such that for each $i = 1,2,\cdots n$, $[\psi(x^k)]_i$ (the $i$-th coordinate of $\psi(x^k)$) $k = 1,2,\cdots, m$ are $m$ distinct real numbers.
	\end{lemma}

	\begin{lemma}
	\label{lemma:n-dim-lemma}
	Suppose $\Fcal$ contains a well function, $x^1,\cdots, x^m$ are distinct points and satisfy the result of Lemma~\ref{lemma:perturbation-argument}, that is, $\{ x^k_i \}$ are $m$ distinct real numbers for any $i$. Then given any $\varepsilon > 0$ for $m$ target points $y^1, \cdots, y^m$, we have a flow map $\psi \in {\Acal}_{\Fcal}$ such that $|\psi(x^k) - y^k| \le \varepsilon$ .
	\end{lemma}

	Now we prove these two lemmas.
    \begin{proof}[Proof of Lemma~\ref{lemma:perturbation-argument}]
    To prove the lemma it is enough to show that if there is a pair of two points $x^j$ and $x^k$, such that $x_I^j = x_I^k$ for some $I$, we can then find a flow map $\eta \in {\Acal}_{\Fcal}$ such that $\eta$ can separate $x_I^j$ and $x_I^k$ and at the same time, do not cause other pairs of points without initially distinct coordinates to overlap. Without loss of generality, we assume $j=1,k=2,I=1$, and we only need to show that if $x^1_1 = x^2_1$, then there exists an $\eta \in {\Acal}_{\Fcal}$ such that
	\begin{enumerate}
	    \item $|\eta(x^k) - x^k| \le \varepsilon_1 := \frac{1}{nm^2}\varepsilon$;
	    \item $[\eta(x^1)]_1 \neq [\eta(x^2)]_1$;
	    \item if $x^k_1 \neq x^l_1$, then $[\eta(x^k)]_1 \neq [\eta(x^l)]_1$.
	\end{enumerate}
    We briefly explain these requirements. Consider
    \begin{align}
        X_1 = \{(k,l)~:~1\le k<l\le m,~~x^k_1 = x^l_1\}
    \end{align}
    and
    \begin{align}
        \eta(X_1) = \{(k,l)~:~1\le k<l\le m,~~[\eta(x^k)]_1 = [\eta(x^l)]_1\}.
    \end{align}
	2 and 3 implies that $\#X_1 > \#\eta(X_1) \ge 0$, hence $\#X_1$ is strictly decreasing after $\eta$.

    Denote $d = \min\{|x^k_1 - x^l_1|: x^k_1 \neq x^l_1\}$. Since $x^1$ and $x^2$ are two distinct points, we can find a coordinate index $I(\neq 1)$ such that $x^1_I \neq x^2_I$. Here we assume $x^1_I < x^2_I$. Suppose $f$ in $\mathcal F$ is a well function with zero set $\{\Omega_1\}$. Written in coordinate form, $f$ is given by
    \begin{align}
        f = (f_1, \cdots, \cdots, f_n),
    \end{align}
	where each $f_i : \R^n \to \R$.  Since $\mathcal F$ is translation invariant, we can assume $\Omega_1$ contains 0 without loss of generality.

    Consider the following dynamics
    \begin{align}
        \dot{z}_1 = f_1(x_I + b), \qquad \dot{z}_i = 0, i = 2,\cdots,n.
    \end{align}
    Notice that the boundedness and convexity of $\Omega_1$ guarantees that the reduced 1D dynamics satisfies our previous discussion (it contains 1D well function, since the intersection of a bounded convex set with a line is an interval)
	In other words, we choose $D = \diag(1,0,\cdots,0)$, $A_{ij} = \delta_{iI}\delta_{jI}$, $b = (0,0, \cdots, b_I, 0)^T$ in the form $\tilde f = Df(A\cdot + b)$. $b_I$ is chosen such that $x_I^1 + b_I \in \Omega_1 $ but $x_I^2 + b_I \not \in \overline{\Omega}_1$. The existence of $b_I$ is implied by  the boundedness of $\Omega_1$.
    We denote by $P_t$  the flow map of this dynamics. We next choose a proper $t$ such that 1,2,3 are satisfied.

	Since $\tilde f_1(x^1) = 0$ and $\tilde f_1(x^2) \neq 0$, we deduce that $[P_t(x^1)]_1 \neq [P_t(x^2)]_1$ whenever $t \neq 0$. Hence 2 is satisfied with no additional condition. Notice that when $|P_t(x^k) - x^k| \le \min(\varepsilon_1, \frac{d}{3})$, then both 1 and 3 are satisfied. Since $ \ch(\{x^k\})$ is bounded, hence $\|P_t - \id\|_{C(\ch(\{x^k\}))} \to 0$ when $t \to 0$ by Proposition~\ref{prop:uniform_conv}. Therefore there exists $t_0 >0$, such that $\|P_t - \id\|_{C(\ch(\{x^k\}))}\le \min(\varepsilon_1, \frac{d}{3})$. Hence we conclude that $\eta = P_{t_0}$ satisfies 1,2,3.
	\end{proof}
	\begin{proof}[Proof of Lemma~\ref{lemma:n-dim-lemma}]
	Without loss of generality, we can assume that for each coordinate index $i$, $y^k_i$ are $m$ distinct real numbers, since if not, we can always add a small perturbation to it directly and this will not affect approximation. We also assume $\Omega_1$ contains origin, as we did in the proof of Lemma~\ref{lemma:perturbation-argument}.

	The basic idea is similar to Lemma~\ref{lemma:perturbation-argument}, by choosing a proper linear transformation we can freeze some point while transporting other points. Since we need to control more than 2 points, we can take multiple transformations and evolve them sequentially. We only need to prove for any coordinate index $i$ (without loss of generality $i = 1$), we can find an $\eta \in {\Acal}_{\Fcal}$ such that $[\eta(x^k)]_1 = y^k_1$.

	With a re-labelling, we can assume that $x_2^1 < x_2^2 < \cdots < x_2^m$.
    Consider the following dynamics
    \begin{align}
        \dot{z}_1 = f_1(az_2 + b), \qquad \dot{z}_i = 0, i = 2,\cdots,n.
    \end{align}
	In other words, we choose $D = \diag(1,0,\cdots,0)$, $A_{ij} = a \delta_{i2}\delta_{j2}$, $b = (0,b_2, \cdots,, 0)^T$ in the form $\tilde f = Df(A\cdot + b)$. $a $ is chosen sufficiently small such that all $Ax^k$ are lying in $\Omega_1$. We denote the flow map by $P_t(b_2)$, where the dependence of $b_2$ is emphasized. To simplify our notation, we use $P_{-t}(b_2)$ to denote the flow map of $\dot{z} = -Df(Az+b)$.

	Now we wish to choose $r^1, r^2, \cdots r^m$, such that $f_1(Ax^i+r^j)=0$ if and only if $i<j$. Let $\{0\} \times (u_l,u_r) \times \cdots  = \Omega_1 \cap \big(\{0\} \times \R \times \cdots )$, where $(u_l,u_r)$ be the restriction of $\Omega_1$, on coordinate index 2. Then a choice of $r^k$ is $r^k = u_r - ax^k_2 + \frac{a}{2}\min_j(x^j_2-x^{j-1}_2).$

    Now $\{ \eta^{(k)} \}$ are defined recursively. That is,
    \begin{align}
        \begin{split}
            &\eta^{(0)} = \id \\
            &\eta^{(k)} = P_{t^k}(r^k)\circ\eta^{(k-1)},
            \quad \text{where} \quad t^k = (y^k_1- [\eta^{(k-1)}(x^k)]_1)/f_1(Ax^k+r^k).
        \end{split}
    \end{align}
	We now prove that $\eta = \eta^{(m)} \circ\dots\circ \eta^{(1)}$ satisfies our requirement. By induction (on $k$), we show that: $[\eta^{(k)}(x^{i})]_1 =y^i_1$ for $i\le k$.
    $k = 0$ is vacuous. Suppose $[\eta^{(k-1)}(x^{i})]_1 =y^i_1$ for $i\le k-1$, since  $[\eta^{(k-1)}(x^{i})]_2 = x^{i}_2$, we have
    \begin{align}
        \eta^{(k)}(x^{i}) = P_{t^k}(r^k)(\eta^{(k-1)}(x^{i})) = \eta^{(k-1)}(x^{i}) = y_i.
    \end{align}
	By definition we know that $\eta^{(k)}(x^k) = y^k$. Hence the induction step is proved. From induction, we know that $\eta$ satisfies our requirement.
\end{proof}

\subsubsection{Proof of Proposition~\ref{prop:high-dim-result}.}
\begin{proof}
    Since $K$ is compact, by extension it suffices to consider the case that $K$ is a hyper-cube. We can for simplicity take the unit hyper-cube $K=[0,1]^{n}$, since the general case is similar.
    Since $\varphi \in L^p(K)$, by standard approximation theory $\varphi$ can be approximated by piecewise constant functions, i.e. there exists
    \begin{equation}
	    \hat{\varphi} = \sum_{\i} \varphi_{\i} \chi_{\i}
    \end{equation}
    such that $\| \hat{\varphi} - \varphi \|_{L^p(K)} \leq \varepsilon/2$.
	Here $\i = [i_1, \cdots, i_n]$ is a multi-index, $\varphi_{\i} \in \R^n$ and $\chi_{\i}$ is the indicator of the cube
    \begin{align}
        \square_{\i} = \left[\frac{i_1}{N} ,\frac{i_1+1}{N}\right] \times \cdots \times \left[\frac{i_n}{N} ,\frac{i_n+1}{N}\right].
    \end{align}
    we also denote $p_{\i} = (\frac{i_1}{N}, \cdots, \frac{i_n}{N})$.
    We also define a shrunken cube
    \begin{align}
        \square_{\i}^{\alpha} = \left[\frac{i_1}{N} ,\frac{i_1+\alpha}{N}\right] \times \cdots \times \left[\frac{i_n}{N} ,\frac{i_n+\alpha}{N}\right].
    \end{align}
    where $0< \alpha \leq 1$. We have $K = \cup_{\i} \square_{\i}$, and we define
    $K^{\alpha} = \cup_{\i}\square_{\i}^{\alpha}$. We also construct a shrinking function in one dimension $h^{\alpha} : [0, 1] \rightarrow [0,1]$, such that $h^{\alpha}(x) = \frac{i}{N}$ if $\frac{i}{N} \le x \le  \frac{i+\alpha}{N}$, and continuously increasing in $[0,1]$.
    Using this, we can form a $n$ dimensional shrinking map by tensor product:
    \begin{align}
        H^{\alpha}(x) = (h^{\alpha}(x_1), \cdots, h^{\alpha}(x_n)).
    \end{align}

    The idea of the proof of Proposition \ref{prop:high-dim-result} is quite simple: we just contract each grid $\square_{\i}$ into a point $  p_{\i}$ approximately, then use the lemma above to transform each $  p_{\i}$ into $\varphi_{\i}$. The latter is discussed in the preliminary step, we here construct an ``almost'' contraction mapping in ${\Acal}_{\Fcal}$ that approximates $H^{\alpha}$.

    \textbf{Claim:} For a given tolerance $\varepsilon_1>0$, there exists a flow map $\widetilde H \in {\Acal}_{\Fcal}$ such that $|\widetilde H - H^{\alpha}|_{C(K)} \le \varepsilon_1$.

    \begin{proof}
    [Proof of the Claim]
    Since $h$ is increasing and continuous, we wish to utilize our result in 1 dimension. Concretely speaking, we demonstrate how to restrict the $n$ dimensional control family to one dimension.

    Suppose $\Fcal$ is a $n$ dimensional control family, then we define for each $f=(f_1,\dots,f_n)\in\Fcal$ the dynamics driven by its restriction to first coordinate by
    \begin{equation}
        \dot{z}_1 = f_1(x_1), \qquad \dot{z}_i = 0 \text{ for } i\ge 2,
    \end{equation}
    i.e., take $D = A = \diag(1,0,\cdots,0)$. Such control systems is denoted as $\Fcal_{R,1}$ ($R$ means restriction and $1$ means first coordinate). Clearly $\Fcal_{R,1}$ is closed under composition. Moreover, ${\Acal}_{\Fcal_{R,1}}$ is coincide with the following set
    \begin{equation}
        \Acal_{\Rcal} \times \{\id\} \times \{\id\} \cdots \{\id\},
    \end{equation}
    where $\Rcal$ is a one dimensional control family
    \begin{equation}
        \Rcal = \{g(x)\,|\,g(x) =f_1(x,0,\cdots,0) \quad f\in\Fcal\}.
    \end{equation}
    Since $\Fcal$ contains a well function, so does $\Rcal$. By Proposition~\ref{prop:ua_sufficiency} we can find $\tilde h \in \Rcal$ such that $|\tilde h - h^{\alpha}|_{C([0,1])} \le \frac{\varepsilon}{n}$.

    By composition we know that $\tilde H := (\tilde h, \tilde h, \cdots, \tilde h)$ is in ${\Acal}_{\Fcal}$, and $|\tilde H - H^{\alpha}|_{C(K)} \le \varepsilon$.
    \end{proof}

    We use aforementioned notations, $\psi$ for transport $p_{\i}$ to $\varphi_{\i}$, and $\tilde H$ is the approximate contraction mapping, satisfying the following estimates:
    \begin{equation}
        |\psi(p_{\i}) - \varphi_{\i}| \le \varepsilon_1 < 1,
    \end{equation}
    \begin{equation}
        \|\tilde H - H^{\alpha}\|_{C(K)} \le \varepsilon_2 <1.
    \end{equation}
    Here $\varepsilon_1$ and $\varepsilon_2$ is to be determined later.

    Now we estimate the error of $\|\psi \circ \widetilde H - \varphi\|_{L^p(K)}$.
    For any $\alpha$ we can write
    \begin{equation}
    \begin{split}
    \|\psi \circ \widetilde H - \varphi\|_{L^p(K)} ~\le~ &  \|\psi \circ \widetilde H- \hat{\varphi}\|_{L^p(K)} + \|\hat{\varphi} - \varphi\|_{L^p(K)}  \\
   ~\le~ & \|\psi \circ \widetilde H - \hat{\varphi}\|_{L^p(K^{\alpha})} + \|\psi \circ \widetilde H - \hat{\varphi}\|_{L^p(K \setminus K^{\alpha})} + \frac{\varepsilon}{2} \\~ \le~ & J_1 + J_2 + \frac{\varepsilon}{2}.
    \end{split}
    \end{equation}

    \textbf{Estimation of $J_1$.}
    \begin{equation}
    \label{eq:I}
    \begin{split}
    J_1  =
    &   \|\psi \circ \widetilde H - \psi \circ H^{\alpha} \|_{L^p(K^{\alpha})} + \| \psi \circ H^{\alpha} - \hat{\varphi} \|_{L^p(K^{\alpha})} \\ \le &
    \omega_{\psi}(\|\tilde H - H^{\alpha}\|_{C(K)}) +
    \sum_{\i} |\psi(p_\i) - \varphi_\i|\cdot|{\square_{\i}}|^{1/p}
    \\ \le &
    \omega_{\psi}(\|\tilde H - H^{\alpha}\|_{C(K)}) + N^{n-n/p}\varepsilon_1 \\ \le &
    \omega_{\psi}(\varepsilon_2) + N^{n-n/p}\varepsilon_1.
    \end{split}
    \end{equation}

    \textbf{Estimation of $J_2$.}
    Denote $\widetilde K = [-1,2]^n$ as an enlarged cube. We have
    \begin{equation}
    \label{eq:J}
    \begin{split}
        J_2
        \le & |K\setminus K^{\alpha}|\cdot\bigg(\diam (\psi(\tilde K))
        + \|\hat{\varphi}\|_{L^{\infty}(K)} \bigg).
    \end{split}
    \end{equation}

We choose $\psi$  such that $\varepsilon_1 \le \frac{\varepsilon}{8}$, $\alpha$ such that
\begin{align}
    |K \setminus K^{\alpha}| \le \bigg(\diam (\psi(\tilde K)) + \|\hat{\varphi}\|_{L^{\infty}(K)} \bigg)^{-1} \frac{\varepsilon}{4},
\end{align}
and finally $\tilde H$ such that $\omega_{\psi}(\varepsilon_2) \le \frac{\varepsilon}{8}$.
Therefore we have $\|\psi \circ \widetilde H - \varphi\| \le \varepsilon$, take $\tilde \varphi = \psi \circ \widetilde H \in {\Acal}_{\Fcal}$ yielding the result.
\end{proof}

As in the 1D case, Proposition~\ref{prop:high-dim-result} together with Corollary~\ref{cor:transformation} imply Theorem~\ref{thm:main}.

\subsection{Approximation Results in Tensor-Product Type Dynamical Systems}

Sometimes, we are interested in control families generated by a tensor products. Such control families have the advantage that it can be parameterized by scalar functions of one variable, hence may allow for greater flexibility in analysis and practice. In this last section, we give some results that apply specifically to tensor product control families. Let us denote
\begin{equation}
    \Pi \Fcal = \{ f(x) = (g(x_1), g(x_2), \cdots, g(x_n)): g \in \Fcal\},
\end{equation}
where $\Fcal$ is a one dimensional control family.

As in the results in higher dimensions, we may wish consider the $n$ dimensional control family $\Fcal(n)$, which is the smallest set containing $\Pi\Fcal$ that is also invariant under $f(\cdot) \mapsto Df(A\cdot + b)$, where $D, A$ are diagonal matrices. However,
all functions $\psi$ in ${\Acal}_{\Fcal}$ are separable: $\psi = (\psi_1(x_1), \cdots, \psi_n(x_n))$. Moreover, we can deduce from the 1D results that $\psi_i$ is continuously increasing. Clearly, this $\Fcal(n)$ has limited approximation ability. Instead, we will relax the requirement that $A$ is diagonal so that it can be any matrix, leading to a stronger version of the restricted invariance requirement in Theorem~\ref{thm:main}. This then leads to the following approximation result:

\begin{proposition}
\label{prop:tensor-product}
Suppose $\Fcal(n)$ and $\Fcal$ satisfies
\begin{enumerate}
    \item $\Fcal(n) $ contains $\Pi \Fcal$;
    \item $\mathcal F(n)$ is invariant under $f(\cdot) \mapsto Df(A\cdot + b)$, where $D$ is any diagonal matrix, $A$ is any matrix, $b$ is any $\R^d$ vectors;

    \item $\overline{{\Acal}_{\Fcal}}$ contains all continuous and increasing functions from $\R$ to $\R$.

\end{enumerate}
Then for any $\varepsilon > 0, p \in (1,\infty]$ and compact set $K$, and $\varphi \in C(\R^n)$, we have $\tilde \varphi \in {\Acal}_{\Fcal(n)}$, such that $\|\varphi - \tilde \varphi\|_{L^p(K)} \le \varepsilon$.
\end{proposition}

\begin{remark}
  This result is not a corollary of Proposition~\ref{prop:high-dim-result}, even if $\chbar{(\Fcal)}$ contains a well function. Since in this case the zero set of the tensor product of the well function may be unbounded.
\end{remark}

Similar to estimation of Proposition~\ref{prop:high-dim-result}, we omit the main body of this proof but only restate preparations  about $\psi$ and $\widetilde H$.

\paragraph{Estimations on $\widetilde H$.}
\begin{lemma}
\label{lemma:tensor-contract}
    Suppose $\Fcal(n)$ and $\Fcal$ satisfies conditions in Proposition~\ref{prop:tensor-product}. For a given tolerance $\varepsilon_1>0$, there exists a flow map $\widetilde H \in {\Acal}_{\Fcal(n)}$ such that $|\widetilde H - H^{\alpha}|_{C(K)} \le \varepsilon_1$.
\end{lemma}
\begin{proof}
This is straightforward from the definition of the tensor product control family and $\overline{{\Acal}_\Fcal}$ contains all continuous increasing functions.
\end{proof}

\paragraph{Estimations on $\psi$.}
Before the main estimate, we first show a useful lemma.

\begin{lemma}
    \label{lemma:add-ci}
    Suppose that $g \in {\Acal}_{\Fcal}$, then $(x_1, x_2, \cdots, x_n) \mapsto (x_1 + g(x_2), x_2 \cdots, x_n)$ is in ${\Acal}_{\Fcal(n)}$.
\end{lemma}
\begin{proof}
We decompose the construction into two parts.
By setting $D = \diag(1,1,0,\cdots,0)$ and $A_{22} = A_{12} = 1$ and $A_{ij} = 0$ otherwise. We first only look at the second coordinate, knowing that for all $g \in {\Acal}_\Fcal$, there exists a finite number $S\geq 1$ of flow maps $\varphi^{t_s}_{f_s}$, $s =1,\cdots, S$, such that
\begin{align}
    g = \varphi^{t_S}_{f_S} \circ \cdots \circ \varphi^{t_1}_{f_1}.
\end{align}
Here, $D$ and $A$ are chosen in the previous paragraph, we deduce that $x_1 - x_2$ is constant under all mapping with following form:
\begin{align}
    z \mapsto Df(Az+b).
\end{align}
If we select $f_1, \cdots, f_S$, sending $x_2$ to $g(x_2)$, we know $x_1 \mapsto x_1 + g(x_2) - x_2$. Hence
\begin{align}
    (x_1,x_2,\cdots, x_n) \mapsto (x_1 + g(x_2),x_2,\cdots,x_n)
\end{align}
is in ${\Acal}_{\Fcal(n)}$.

Also, by setting $D = \diag(1,0,\cdots,0)$ and $A_{ij} = \delta_{i2} \delta_{j2}$, we know that
\begin{align}
    (x_1,x_2,\cdots, x_n) \mapsto (x_1,g^{-1}(x_2),\cdots,x_n)
\end{align}
is in ${\Acal}_{\Fcal(n)}$. Composing two parts yields the result.
\end{proof}

\begin{lemma}[Analogous to Lemma~\ref{lemma:perturbation-argument}]
	Suppose $\Fcal(n)$ and $\Fcal$ satisfies conditions in Proposition~\ref{prop:tensor-product} and $x^1,\cdots, x^m$ are distinct points. Then, given any $\varepsilon >0$, there exists a flow map $\psi \in {\Acal}_{\Fcal(n)}$, such that $|\psi(x^k) - x^k| \le \varepsilon$, and that for each $i = 1,2,\cdots n$, $[\psi(x^k)]_i$ (the $i^{\text{th}}$ coordinate of $\psi(x^k)$), $k = 1,2,\cdots, m$ are $m$ distinct real numbers.
	\end{lemma}

	\begin{proof}
	Similar to Lemma~\ref{lemma:perturbation-argument}, we prove that if $x_1^1$ and $x_1^2$, then we can find a $\eta$ that separates them. The three requirements are the same as we established in Lemma~\ref{lemma:perturbation-argument}, hence omitted here.

    Suppose $x_I^1 \neq x_I^2$ for some $I$. We can find two continuously increasing function $P(\cdot)$ and $Q(\cdot)$ such that
    \begin{align}
        \begin{split}
            P(x_I^1) -Q(x_I^1)&= 1, \\
            P(x_I^2) -Q(x_I^2)&= -1, \\
            P(x_I^k) -Q(x_I^k)&= 0,
        \end{split}
    \end{align}
    for the other $k$'s. By assumptions on $\Fcal$, we can find $\widetilde P(\cdot)$ and $\widetilde Q(\cdot)$ such that
    \begin{align}
        \|P(\cdot) - \widetilde P(\cdot) \|_{C([0,1])}, \,\, \|Q(\cdot) - \widetilde Q(\cdot)\|_{C([0,1])} \le \frac{\min(d,1)}{4}.
    \end{align}

By Lemma~\ref{lemma:tensor-contract}, we know that $(x_1,\cdots, x_n) \mapsto (x_1 + \widetilde P(x_I) - \widetilde Q(x_I), \cdots x_n)$ is in ${\Acal}_{\Fcal(n)}$. It can be checked that this is our desired $\eta$.
	\end{proof}

	\begin{lemma}[Analogous to Lemma~\ref{lemma:n-dim-lemma}]
	Suppose $\Fcal(n)$ and $\Fcal$ satisfies conditions in Proposition~\ref{prop:tensor-product}and $x^1,\cdots, x^m$ are distinct points. Moreover we assume $x_k$ satisfies the result of Lemma~\ref{lemma:perturbation-argument}, that is, $\{ x^k_i, k=1,\dots,m\}$ are $m$ distinct real numbers for any $i$. Then, given any $\varepsilon > 0$, for $m$ target point $y^1, \cdots, y^m$, we have a flow map $\psi \in {\Acal}_{\Fcal(n)}$ such that $|\psi(x^k) - y^k| \le \varepsilon$ .
	\end{lemma}

\begin{proof}
Similar to proof of Lemma~\ref{lemma:n-dim-lemma}, we use $x_2$ to translate $x_1$ (denoted as $\eta$). We find two one dimensional $P(\cdot)$ and $Q(\cdot)$, both continuously increasing, such that $x^k_1 + P(x^k_2) - Q(x^k_2) = y^k_1$. By assumptions on $\Fcal$ we can find $\widetilde P(\cdot)$ and $\widetilde Q(\cdot)$, such that
\begin{align}
    \|P(\cdot) - \widetilde P(\cdot) \|_{C([0,1])}, \,\, \|Q(\cdot) - \widetilde Q(\cdot)\|_{C([0,1])} \le \frac{\varepsilon}{2}.
\end{align}

Since
\begin{align}
    \eta = (x_1,\cdots, x_n) \mapsto (x_1 + \widetilde P(x_2) - \widetilde Q(x_2), \cdots x_n)
\end{align}
is in ${\Acal}_{\Fcal(n)}$, and $|[\eta(x^k)]_1 - y_1^k| \le \varepsilon$, we conclude that $\eta$ satisfies our requirement.
\end{proof}

Combining these two lemmas, we obtain the following result from which we can deduce Prop.~\ref{prop:tensor-product}
\begin{lemma}[Analogous to Lemma~\ref{lemma:high-dim}]
		\label{prop:high-dim-tensor}
		Suppose $\Fcal(n)$ and $\Fcal$ satisfy the conditions in Proposition~\ref{prop:tensor-product}. Let $\varepsilon>0$ and $x^1, \cdots, x^m, y^1, \cdots, y^m \in \R^n$ be such that $\{x^k\}$ are distinct points. Then there exists $\psi \in A_{\mathcal F(n)}$ such that $|\psi(x^k)-y^k| \le \varepsilon$ for all $k = 1,\dots,m$.
\end{lemma}

\bibliographystyle{apa}
\bibliography{library}

\end{document}